%% file: neurips_2026.tex
\documentclass{article}

\usepackage[preprint]{neurips_2026}

\usepackage[utf8]{inputenc} 
\usepackage[T1]{fontenc}    
\usepackage{hyperref}       
\usepackage{url}            
\usepackage{booktabs}       
\usepackage{amsfonts}       
\usepackage{nicefrac}       
\usepackage{microtype}      

\usepackage{graphicx}
\usepackage{subcaption}
\usepackage{multirow}
\usepackage{multicol}
\usepackage{enumitem}
\usepackage{makecell}

\usepackage{tabularx}
\usepackage{array}

\usepackage{algorithm}
\usepackage{algorithmic}

\usepackage{amsmath}
\usepackage{amssymb}
\usepackage{mathtools}
\usepackage{amsthm}
\usepackage{enumitem}
\usepackage{wrapfig}
\usepackage{makecell}
\usepackage{xcolor}
\usepackage[table]{xcolor}
\definecolor{lightblue}{HTML}{E6F3FD}

\usepackage[capitalize,noabbrev]{cleveref}

\theoremstyle{plain}
\newtheorem{theorem}{Theorem}[section]

\theoremstyle{definition}

\theoremstyle{remark}

\usepackage[textsize=tiny]{todonotes}

\title{Learning to Configure Agentic AI Systems}

\title{Learning to Configure Agentic AI Systems}

\author{%
Aditya Taparia\thanks{Equal contribution.} \quad
Som Sagar\footnotemark[1] \quad
Ransalu Senanayake \\
School of Computing and Augmented Intelligence \\
Arizona State University \\
Tempe, AZ, USA \\
\texttt{\{ataparia, ssagar6, ransalu\}@asu.edu}
}

\begin{document}

\maketitle


\begin{abstract}
Configuring LLM-based agent systems involves choosing workflows, tools, token budgets, and prompts from a large combinatorial design space, and is typically handled today by fixed templates or hand-tuned heuristics that apply the same configuration regardless of query difficulty, leading to brittle behavior and wasted compute. To address this, we formulate \emph{agent configuration} as a semi-Markov decision process (SMDP) where each configuration acts as a temporally extended option that determines how an agent system processes a query, and introduce \textbf{ARC} (\textbf{A}gentic \textbf{R}esource \& \textbf{C}onfiguration learner), a lightweight hierarchical policy that dynamically selects query-specific agent configurations. Across reasoning, tool-use, and agentic benchmarks, ARC consistently improves over budget-matched tool-augmented LLMs, increasing average reasoning accuracy by 31.3\%, tool-use accuracy by 13.95\%, and doubling $\tau$-Bench (Airline) Pass\textasciicircum1 success from 9.0\% to 18.0\%. These results demonstrate that learning per-query agent configurations is a powerful alternative to ``one size fits all'' designs. \textbf{Codebase:} \href{https://github.com/somsagar07/Context_Optimization}{Github}
\end{abstract}

\input{sections/intro}

\input{sections/related_works}

\input{sections/method_2}

\input{sections/experiments}

\input{sections/conclusion}

\bibliographystyle{plainnat}
\bibliography{example_paper}

\input{sections/appendix}

\end{document}

%% file: sections/intro.tex
\section{Introduction}
Large Language Models (LLMs) have evolved from simple answer predictors to being the backbone of complex \emph{multi-agent systems} capable of iterative planning, tool usage, and multi-step reasoning. In this new paradigm, the performance of an agentic system is governed not only by the underlying LLM but also by the architecture that wraps it: workflows (e.g. voters, verifiers, optimizers), tool availability, information routing, and context management.

Current approaches to agentic architecture design largely rely on static heuristics or ``kitchen sink'' strategies that flood the context window with extensive history and retrieved evidence. However, this is suboptimal for two reasons. First, performance degrades in long contexts due to the ``lost-in-the-middle'' phenomenon, where models fail to attend to relevant information~\cite{liu2024lost, hong2025context}. Second, static systems are inefficient; they fail to adapt their computational footprint to task difficulty. For example, an agentic system designed for multi-hop reasoning might trigger expensive web-search tools and iterative verification loops even for a basic arithmetic query. This ``one-size-fits-all'' approach leads to wasted compute and unnecessary latency by applying the same heavy resources to trivial and complex inputs.

\begin{figure}[t]
\begin{center}
\includegraphics[width=\linewidth]{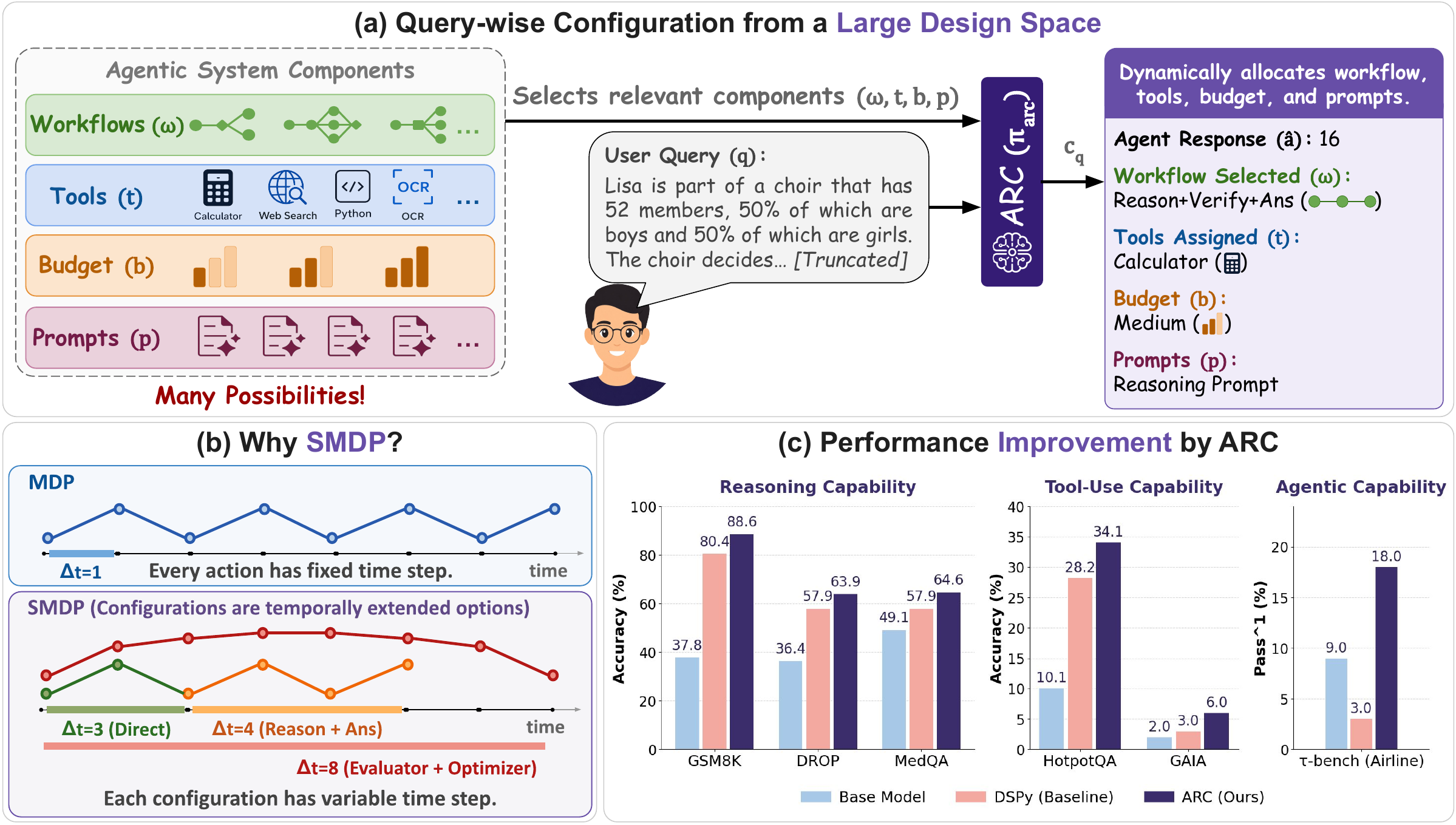}
\caption{
\textbf{Overview of ARC.} (a) ARC learns to select a query-specific agent configuration from a large combinatorial design space. (b) Unlike standard MDP actions with fixed duration, agent configurations induce temporally extended executions with variable numbers of LLM calls, motivating our SMDP formulation. (c) Across reasoning, tool-use, and agentic benchmarks, ARC improves performance over budget-matched base models and strong baselines for Qwen 2.5 7B.
}
\label{fig:intro}
\end{center}
\vspace{-1.5em}
\end{figure}

Ideally, an agent system should \emph{adapt its configuration to each query at hand}. While intuitive, learning such dynamic configurations is difficult because the design space is combinatorial. For instance, consider a system that must answer questions of varying difficulty: Should it invoke a single-shot Chain-of-Thought? Deploy parallel voters to cross-check answers? Allocate a generous token budget, or conserve compute with a minimal response? Even a simple 3-agent system with 5 workflow patterns, 3 independently selectable tools for two agents, and 3 budget levels per agent yields $5 \times (2^3)^2 \times 3^3 = 8{,}640$ structural configurations before prompt selection. Adding only 20 prompt choices already expands this to over $10^5$ possible configurations, rendering brute-force search intractable and manual tuning impractical.

To address this, we cast agent configuration as a query-wise decision making problem. A key insight of our work is that agent workflows are not atomic actions, but rather structured processes with varying computational footprints. For instance, a simple inference strategy, such as direct prompting, consumes a single LLM invocation. In contrast, iterative paradigms like an Evaluator-Optimizer loop~\cite{madaan2023self} may require multiple sequential calls and variable cost. Standard MDP formulations fall short here because they assume actions have a fixed, discrete duration. We argue that this is fundamentally the wrong abstraction for learning to chose agentic configurations. Instead, we model each configuration as an \emph{option}~\cite{sutton1999between} within a SMDP.
This formulation naturally handles actions with variable durations and costs, and ultimately enables a lightweight controller on top of a frozen agent system to jointly reason about \emph{what} to deploy and \emph{how much compute it will consume}, balancing task accuracy against efficiency in a principled way.

Our contributions are as follows:
\begin{enumerate}[leftmargin=*, noitemsep, topsep=0pt]
    \item We formulate agent configuration as SMDP (see Figure~\ref{fig:intro}b), and introduce ARC, a \textbf{Hierarchical Reinforcement Learning (HRL) framework} where a high-level policy selects workflows, tools and, budget while a low-level policy composes prompt instructions. Importantly, the learning operates over a lightweight network, without updating the backbone LLM.
    \item We propose a \textbf{hybrid training pipeline} that combines masked RL with supervised fine-tuning (SFT) on elite trajectories to stabilize learning under sparse rewards.
    \item We demonstrate through \textbf{extensive experiments} across reasoning, tool-use, and agentic benchmarks that our method achieves significant gains while optimizing cost over static, flat RL-policies and configuration optimization baselines (see Figure~\ref{fig:intro}c).
\end{enumerate}

%% file: sections/related_works.tex
\vspace{-0.8em}
\section{Related Work}
\vspace{-0.5em}
We situate our work at the intersection of three research threads: LLM-based agents, prompt and workflow optimization, and hierarchical reinforcement learning.

\textbf{Large Language Agents.} LLM-based agents extend language models with iterative decision-making, tool use, and multi-step interaction~\cite{yao2022react, schick2023toolformer}. Recent frameworks provide abstractions for tool calling and multi-agent coordination~\cite{Chase2022LangChain, wu2024autogen}, while evaluation suites stress realistic long-horizon behavior across general agent benchmarks and web-interaction settings~\cite{liu2023agentbench, zhou2023webarena}. Tool-centric datasets further broaden the space of available actions and APIs~\cite{qin2023toolllm}, and agents have been instantiated in software engineering interfaces where actions include editing code and navigating repositories~\cite{yang2024swe}. Agent orchestration also extends to physical robot teams, where LLM-generated plans are grounded in executable skills and validated against state and workflow constraints~\cite{liu2026physical}. Despite this progress, most systems rely on manually specified or heuristically tuned workflows with fixed sequencing of reasoning, retrieval, and tool calls. In contrast, our work treats agent configuration as a learnable decision problem, enabling query-wise adaptive workflow and budget selection under resource constraints.

\textbf{Prompt and Workflow Optimization.} Early work on adapting LLM behavior focused on the input prompt optimization. Automated frameworks such as OPRO~\cite{yang2023large}, DSPy~\cite{khattab2023dspy}, and GEPA~\cite{agrawal2025gepa} treat the prompt as an optimization variable, iteratively searching for improved instructions. More recently, the optimization target has shifted from static prompts to agentic workflows and context management. Methods like LLMLingua~\cite{jiang2023llmlingua} and LongLLMLingua~\cite{jiang2024longllmlingua} compress input prompts to fit context windows, while cost-aware frameworks have evolved from model routing~\cite{chen2023frugalgpt} to token-budget-aware reasoning~\cite{han2025token}. Our work unifies these directions by learning a hierarchical policy that jointly optimizes workflow structure and context budget end-to-end.

\textbf{Hierarchical Planning and Control.} 
While LLM agents are typically viewed through the lens of prompt engineering, our work connects to the rich literature on HRL. 
Classic frameworks such as the Options framework~\cite{sutton1999between} and Feudal RL~\cite{dayan1992feudal} decompose complex tasks into high-level goals and low-level actions, enabling more efficient exploration in large state-action spaces. 
Recent work has applied HRL to language model fine-tuning and multi-step reasoning~\cite{pang2024iterative}, though typically within fixed architectural constraints. 
In contrast, we show that formulating agent configuration as an SMDP naturally induces an HRL structure. A high-level policy selects temporally extended configurations over workflows, tools, and budgets, while a lower-level policy composes prompt instructions. 
This formulation enables structured search over a combinatorial agent design space that flat RL approaches cannot efficiently navigate.

%% file: sections/method_2.tex
\section{Methodology}
\vspace{-0.5em}
\label{sec:methodology}

We aim to learn a policy $\pi_{arc}$ that adapts agent configurations to each input query, balancing correctness against computational cost. Formally, let $\mathcal{D} = \{(q_i, a_i)\}_{i=1}^{N}$ be a dataset of queries $q_i$ and ground-truth answers $a_i$. As shown in Figure~\ref{fig:intro}a, for each query $q$ in a session, $\pi_{arc}$ selects a configuration $c_q = (\omega, \mathbf{t}, \mathbf{b}, \mathbf{p})$, where $\omega$ represent how ``LLMs are connected'', $\mathbf{t}$ encodes the enabled tools, $\mathbf{b}$ denotes token-budget tiers, and $\mathbf{p}$ specifies prompt instructions for each agent. Executing the configured system on $q$ produces a response $\hat{a}$, a correctness signal $\mathbb{I}[\hat{a}=a]$, and cost statistics. Through which we define the per-turn utility as $U(q, c) = \mathbb{I}[\hat{a}=a] - \lambda\, C_{\text{cost}}(c)$, where $C_{\text{cost}}(c)$ aggregates token usage and runtime, and $\lambda \geq 0$ controls the accuracy efficiency trade-off.
 
The challenge here is that different workflows impose fundamentally different computational commitments: a direct call consumes one LLM invocation, while an evaluator-optimizer (Appendix~\ref{app:workflows}) loop may require up to seven. A standard MDP treats all configuration choices as instantaneous actions and cannot represent this variation. Therefore, we model agent configurations as an \emph{option}~\citep{sutton1999between} i.e., a temporally extended action whose duration depends on the chosen workflow. This yields a semi-Markov decision process (SMDP) in which the controller reasons about both what to deploy and what temporal cost each configuration commits the session to. 

\subsection{Modeling the Decision Process over Options}
\label{sec:smdp}
\vspace{-0.5em}
The agent configuration problem is a semi-Markov decision process $\mathcal{M} = (\mathcal{S}, \mathcal{O}, P, R, \gamma)$ where:
\begin{itemize}[leftmargin=*, noitemsep, topsep=-1pt]
 
\item \textbf{State space $(\mathcal{S})$}: For query $q$, the state $s_q = [\phi(q);\, \mathbf{f}_q]$ concatenates a semantic embedding $\phi(q)$ (MetaCLIP-H/14; \citet{xu2023demystifying}) with hand-crafted features $\mathbf{f}_q$ encoding query length, numerical density, and binary indicators for multi-step reasoning and tool use. See Appendix~\ref{app:metaclip} for the embedding selection ablation.

\item \textbf{Options $(\mathcal{O})$}: Each option $o = (\mathcal{I}_o, \pi_{arc}, \beta_o)$ represents a complete act of configuring and executing the agent system on a query.:
  \begin{itemize}[leftmargin=1.5em, noitemsep, topsep=1pt]
    \item \emph{Initiation set} $\mathcal{I}_o \subseteq \mathcal{S}$: the set of states in which option $o$ may be invoked. After action masking prunes structurally invalid configurations (Section~\ref{sec:training}), all surviving options share $\mathcal{I}_o = \mathcal{S}$.
    \item \emph{Policy} $\pi_{arc}$: given state $s_q$, selects a structural configuration $(\omega, \mathbf{t}, \mathbf{b})$ and autoregressively composes prompt instructions $\mathbf{p}$ from a library of semantic fragments $\mathcal{P}$, yielding the full configuration $c = (\omega, \mathbf{t}, \mathbf{b}, \mathbf{p})$.
    \item \emph{Termination} $\beta_o$: deterministic for fixed-structure workflows (e.g., Direct terminates after one LLM call, Reason+Answer after two) and stochastic for iterative workflows (e.g., Evaluator-Optimizer terminates upon convergence or at $K_{\max}$ iterations).
  \end{itemize}
 
\item \textbf{Transition dynamics} $P(s', k \mid s, o)$: joint transition-duration kernel giving the probability of reaching state $s'$ after $k$ steps under option $o$ from state $s$. Executing option $o$ in state $s_q$ runs the full workflow, consuming $k_o \in \{1, \ldots, K_{\max}\}$ LLM calls and $n_{\text{tokens}}(o)$ tokens.
 
\item \textbf{Reward} $R(s, o)$: scalar reward received upon option termination, incorporating correctness and duration-dependent costs (Eq.~\ref{eq:reward}).
 
\item \textbf{Discount factor} $\gamma = 0.9$: applied per primitive step ($\gamma^{k_o}$), so heavier workflows incur steeper discounting on subsequent turns.
 
\end{itemize}

The combinatorial size of the configurations makes direct optimization of $\pi_{arc}$ intractable. Therefore, $\pi_{arc}$ is decomposed into 1) a \emph{structure policy} $\pi_{struct}$ that selects the workflow, tools, and budget, and 2) a \emph{prompt policy} $\pi_{prompt}$ that composes instructions conditioned on the structural choice:
\begin{equation}
\pi_{arc}(c \mid s_q) \;=\; \pi_{\text{struct}}(\omega, \mathbf{t}, \mathbf{b} \mid s_q) \;\cdot\;  \pi_{\text{prompt}}(p \mid s_q, \omega, \mathbf{t}, \mathbf{b})
\label{eq:meta_policy}
\end{equation}This hierarchy replaces a single joint decision over $|\mathcal{O}|$ configurations with sequential structural and prompt decisions. Action masking (Section~\ref{sec:training}) further prunes invalid branches during training.

\textbf{Structure Policy.} The structure policy $\pi_{\text{struct}}$ selects one option per query based on the current state $s_q$, where each option corresponds to a composite action $(\omega, \mathbf{t}, \mathbf{b})$. Workflows are drawn from established agentic patterns~\citep{anthropic2024agents}, spanning direct, sequential (Chain-of-Thought), parallel (Voting), and iterative (Evaluator-Optimizer) topologies. Tools such as calculators, web search, and code execution can be independently enabled or disabled per agent, and output budget tiers control the token allocation for each agent in the pipeline. 

\textbf{Prompt Policy.} The prompt policy $\pi_{\text{prompt}}$ operates as a sequential decision-making process, conditioned on the structural choice $(\omega, \mathbf{t}, \mathbf{b})$. Its role is to operationalize the selected option by defining the instructions for each agent. At each step, $\pi_{\text{prompt}}$ selects a fragment from $\mathcal{P} \cup \textsc{stop}$. The composed instructions complete the configuration $c = (\omega, \mathbf{t}, \mathbf{b}, \mathbf{p})$, which is then executed. We use a separate discount $\gamma_{\text{prompt}} = 0.5$ for the prompt composition steps: since the reward is observed only after execution, the aggressive discount ensures that credit concentrates on the first few fragment selections and penalizes long, redundant instruction sequences. To ensure high-quality instructions, $\mathcal{P}$ is augmented with dataset-specific fragments generated via meta-prompting. Our ablations (Appendix~\ref{app:best_prompter}) demonstrate that generation using GPT-5.2 outperforms other language models.

\begin{figure*}[t]
\begin{center}
\includegraphics[width=\linewidth]{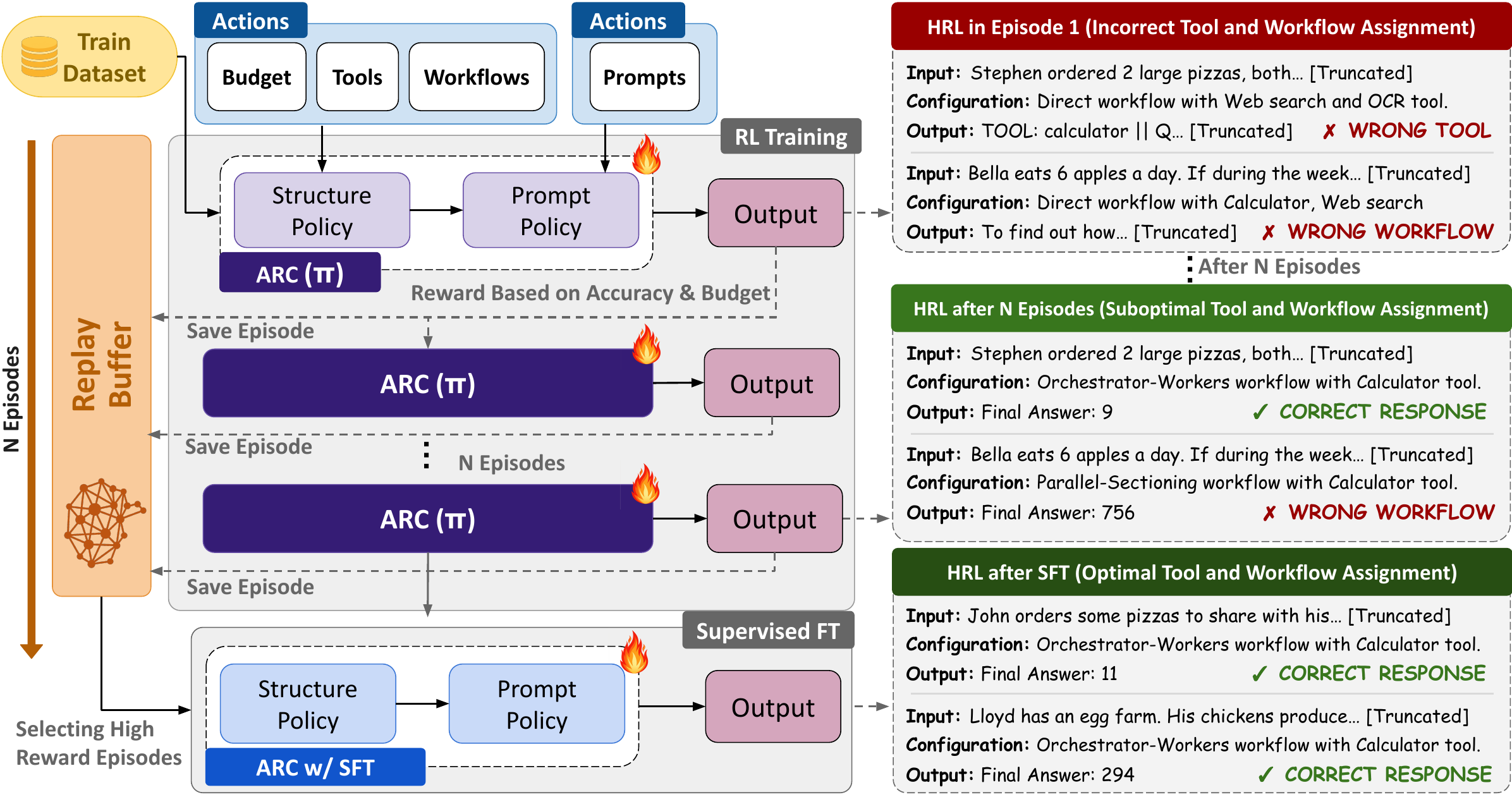}
\caption{\textbf{Training pipeline.} The structure policy selects workflows, tools, and budgets while the prompt policy composes instructions. During RL training, episodes are stored in a memory buffer. After RL converges, high-reward episodes are filtered and used for supervised fine-tuning (SFT), which consolidates successful strategies and improves consistency.}
\label{fig:motivating_examples}
\end{center}
\vspace{-1.8em}
\end{figure*}

\vspace{-0.5em}
\subsection{Training Procedure}
\label{sec:training}
\vspace{-0.5em}

We optimize both policies end-to-end using Proximal Policy Optimization (PPO)~\citep{schulman2017proximal} as shown in Figure~\ref{fig:motivating_examples}. Each policy maintains a separate value network for advantage estimation, and we apply per-batch advantage normalization to stabilize learning across the heterogeneous reward landscape. The following objective is applied independently to each factor of $\pi_{arc}$:
\begin{equation}
\mathcal{L}^{\text{PPO}}(\theta) = \mathbb{E}\!\left[\min\!\left( \frac{\pi(a_t|s_t)}{\pi^{\text{old}}(a_t|s_t)} \hat{A}_t,\;\; \mathrm{clip}\!\left(\frac{\pi(a_t|s_t)}{\pi^{\text{old}}(a_t|s_t)}, 1{-}\epsilon, 1{+}\epsilon\right)\hat{A}_t\right)\right],
\label{eq:ppo}
\end{equation}
where $\hat{A}_t$ is the estimated advantage, with $\pi = \pi_{\text{struct}}$ operating over Options and $\pi = \pi_{\text{prompt}}$ operating over prompt fragments. We add entropy regularization to encourage exploration in the early stages of training.
 
\textbf{Reward Design.} Designing rewards for variable-duration actions is non-trivial: optimizing correctness alone leads to over-provisioning (allocating maximum resources), while pure efficiency metrics sacrifice accuracy. We decompose the reward $R$ into three interpretable terms:
\begin{equation}
R(s,o) = \alpha \cdot \mathbb{I}[\text{correct}] \;-\; \beta_s \cdot k_o + \beta_t \cdot {n_{\text{tokens}}}/{T_{\max}} \;+\; \eta \cdot \mathcal{R}_{\text{tool}},
\label{eq:reward}
\end{equation}
where $k_o$ is the number of LLM calls taken by option $o$, $n_{\text{tokens}}$ is total token consumption normalized by the maximum budget $T_{\max}$, and $\alpha, \beta_s, \beta_t, \eta$ are weighting coefficients. The duration penalty $\beta_s \cdot k_o$ penalizes each option proportionally to its actual temporal commitment,for selecting shorter options when they suffice without penalizing complex workflows that require difficult queries.

\begin{wrapfigure}[20]{r}{0.55\columnwidth}
\vspace{-1.0em}
\centering
\footnotesize

\hrule
\vspace{2pt}
\captionof{algorithm}{Training Pipeline (SMDP $\to$ SFT)}
\label{alg:training}
\vspace{-4pt}
\hrule

\begin{minipage}{\linewidth}
\begin{algorithmic}[1]
\vspace{2pt}
\REQUIRE Dataset $\mathcal{D}$, policies $\pi_{\text{struct}}$, $\pi_{\text{prompt}}$, buffer $\mathcal{B} = \emptyset$
\STATE Initialize policies and value networks $V_{\text{struct}}$, $V_{\text{prompt}}$
\FOR{$t = 1$ to $T_{\text{RL}}$}
    \FOR{each $q \sim \mathcal{D}$}
        \STATE Encode $s_q = [\phi(q); \mathbf{f}_q]$; sample $o \sim \pi_{\text{struct}}(\cdot|s_q)$
        \STATE Apply mask based on workflow $\omega$
        \STATE Compose prompts $\mathbf{p}$ via $\pi_{\text{prompt}}$
        \STATE Execute workflow $c = (\omega, \mathbf{t}, \mathbf{b}, \mathbf{p})$ for $k_o$ steps
        \STATE Compute reward $R$; store $(s_q, o, k_o, R)$ in $\mathcal{B}$
    \ENDFOR
    \STATE Compute advantages $\hat{A}_t$; normalize per batch
    \FOR{$e = 1$ to $E$}
        \STATE Update $\theta$ via PPO clipped objective with entropy regularization
    \ENDFOR
\ENDFOR
\STATE Extract elite: $\mathcal{D}_{\text{elite}} = \{(s_i, o_i^*) \in \mathcal{B} : \text{correct}_i \wedge R_i \geq \tau\}$
\FOR{$e = 1$ to $E_{\text{SFT}}$}
    \STATE Update $\theta$ via SFT: $\max \mathbb{E}_{(s, o^*) \sim \mathcal{D}_{\text{elite}}}[\log \pi_{\text{struct}}(o^* | s)]$
\ENDFOR
\RETURN $\pi_{\text{struct}}^{\text{SFT}}$, $\pi_{\text{prompt}}^{\text{SFT}}$
\end{algorithmic}
\end{minipage}
\hrule
\vspace{-1.0em}
\end{wrapfigure}
The tool shaping term $\mathcal{R}_{\text{tool}}$ addresses a key challenge: the structure policy allocates tools, but the downstream LLM decides whether to invoke them. Naively rewarding tool allocation creates a mismatch: tools may be provisioned but never used. We design an asymmetric reward:

\textbf{Action Masking.} The raw option space ($|\mathcal{O}| = 9 \times 16^2 \times 3^3 = 62{,}208$) contains structurally invalid configurations. For example, the Direct workflow employs
a single agent, yet $\mathcal{O}$ permits allocation of tools and budgets to a second agent dimension. We employ action masking to prune such invalid combinations, reducing $|\mathcal{O}|$ to $41{,}904$ valid options (a $32.6\%$ reduction).

During action sampling, once $\pi_{\text{struct}}$ selects a workflow $\omega$, we apply a conditional mask that excludes incompatible choices by setting their logits to $-\infty$ prior to the softmax. Formally, for action dimension $i$ with logits $\mathbf{z}_i$, we apply a mask $\mathbf{m}_i(\omega) \in \{0,1\}^{|\mathcal{A}_i|}$ conditioned on the selected workflow, yielding masked logits $\tilde{\mathbf{z}}_i = \mathbf{z}_i + \log\!\big(\mathbf{m}_i(\omega)\big)$. The same mask is applied during policy updates, ensuring that $\rho_t$ in Eq.~\ref{eq:ppo} is computed over consistent, valid distributions.

 \vspace{-0.5em}
\subsection{Post-Training Refinement}
  \vspace{-0.5em}
Although RL effectively explores the combinatorial design space, high-variance gradient estimates can leave the final policy with residual stochasticity. We introduce a supervised fine-tuning (SFT) refinement phase that runs after the RL completion. \textbf{SFT is computationally inexpensive}: it fine-tunes only the lightweight policy networks (not the LLM) via supervised learning on a subset of the RL buffer, avoiding the expensive rollouts required during RL. We filter $\mathcal{B}$ to obtain elite trajectories that are both correct and exceed a reward threshold: $\mathcal{D}_{\text{elite}} = \{(s_i, o_i^*) \in \mathcal{B} : \text{correct}_i \wedge R(s_i, o_i^*) \geq \tau\},$
where $\tau$ retains the top 30\% of episodes by reward. We fine-tune both policies via maximum likelihood on the elite demonstrations:
$\max_\theta \; \mathbb{E}_{(s, o^*) \sim \mathcal{D}_{\text{elite}}}\!\left[\log \pi_{\text{struct}}(o^* | s)\right].$ This distillation concentrates the policy on configurations that succeeded during training, providing formal guarantees (Theorem~\ref{thm:concentration}).

\subsection{Theoretical Guarantees}
\label{sec:theory}
 
We establish two theoretical results: the SMDP formulation admits a well-defined unique optimal policy that RL can recover (Theorem~\ref{thm:convergence}), and SFT concentrates the final policy on high-reward configurations with a formal performance guarantee (Theorem~\ref{thm:concentration}).
 
\begin{theorem}[SMDP Convergence]
\label{thm:convergence}
Let $\mathcal{T}$ be the SMDP Bellman optimality operator:
\[
(\mathcal{T}Q)(s,o) \;=\; \mathbb{E}\!\left[R(s,o) + \gamma^{k_o} \max_{o' \in \mathcal{O}} Q(s',o') \;\middle|\; s, o\right].
\]
Then: \emph{(i)}~$\mathcal{T}$ is a $\gamma$-contraction in $\ell_\infty$ norm, so a unique fixed point $Q^*$ exists; and \emph{(ii)}~under Robbins--Monro step sizes ($\sum_n \alpha_n = \infty$, $\sum_n \alpha_n^2 < \infty$), infinite visitation of all $(s,o)$ pairs, and bounded durations $1 \leq k_o \leq K_{\max}$, SMDP Q-learning iterates converge: $Q_n \to Q^*$ almost surely.
\end{theorem}
 
\begin{proof}[Proof sketch]
Since every Option executes at least one LLM call ($k_o \geq 1$), we have $\gamma^{k_o} \leq \gamma < 1$. The contraction follows: $\|\mathcal{T}Q_1 - \mathcal{T}Q_2\|_\infty \leq \gamma \|Q_1 - Q_2\|_\infty$. Existence and uniqueness of $Q^*$ is immediate from the Banach fixed-point theorem. Convergence of the stochastic iterates follows from the extension of \cite{jaakkola1993convergence} to SMDPs~\cite{sutton1999between}. Full proof in Appendix~\ref{app:proofs}.
\end{proof}
 
\begin{theorem}[Policy Concentration]
\label{thm:concentration}
Let $\mathcal{C}_{\emph{elite}}(s) = \{c : (s,c) \in \mathcal{D}_{\emph{elite}}\}$. Under sufficient model capacity, the refined policy $\pi_{arc}^{\emph{SFT}}$ satisfies:
\emph{(i)}~$\pi_{arc}^{\emph{SFT}}(c|s) > 0 \Rightarrow c \in \mathcal{C}_{\emph{elite}}(s)$ \emph{(support restriction)}, and
\emph{(ii)}~$\mathbb{E}_{c \sim \pi_{arc}^{\emph{SFT}}}[R(s,c)] \geq \tau$ \emph{(reward floor)}.
\end{theorem}
 
\begin{proof}[Proof sketch]
The MLE objective minimizes $D_{\mathrm{KL}}(\hat{p}_{\text{elite}} \| \pi_{arc})$, attaining zero when $\pi_{arc} = \hat{p}_{\text{elite}}$. Support restriction follows because $\hat{p}_{\text{elite}}(c|s) = 0$ for $c \notin \mathcal{C}_{\text{elite}}(s)$. The reward floor holds because every $(s,c) \in \mathcal{D}_{\text{elite}}$ satisfies $R(s,c) \geq \tau$ by construction. Full proof in Appendix~\ref{app:proofs}.
\end{proof}
 
Theorem~\ref{thm:convergence} establishes that the SMDP has a unique optimal value function $Q^*$, so RL targets a well-defined optimum and the elite buffer contains near-optimal configurations. Theorem~\ref{thm:concentration} guarantees that SFT then concentrates $\pi_{arc}$ onto these configurations, establishing a performance floor at $\tau$. We also validate this empirically in Section~\ref{sec:ablations}

%% file: sections/experiments.tex
\vspace{-0.5em}
\section{Experiments}

Our experiments investigate three key aspects of the proposed framework ARC: \textbf{(RQ1)} Does the learned configuration outperform fixed architectures? \textbf{(RQ2)} Does adaptive allocation improve efficiency? \textbf{(RQ3)} Can these learned configuration be transferred across task and parameters? 

\vspace{-0.5em}
\subsection{Experimental Setup}
\noindent\textbf{Benchmarks.} 
We evaluate on six benchmarks spanning across three primary capability axes: \textit{Reasoning Capability}, which includes GSM8k~\cite{cobbe2021training}, DROP~\cite{dua2019drop}, and MedQA~\cite{jin2020disease}; \textit{Tool-Use Capability}, comprising HotPotQA~\cite{yang2018hotpotqa} and GAIA~\cite{mialon2023gaia}; and \textit{Agentic Capability} on $\tau$-Bench (Airline)~\cite{yao2024tau}. We employ standard training split for policy learning and evaluate on the test (or validation) set. For GAIA, we partitioned the validation set, using the first 65 samples for training and the rest for evaluation. More details are provided in Appendix~\ref{app:training}.

\noindent \textbf{Baselines.} We compare our framework and its non SFT-refined variant against several baselines: (1) \textit{Base models with tools}, utilizing off-the-shelf LLMs evaluated under max-token constraints and restricted to non-iterative tool calls; (2) \textit{Search-based methods}, employing Grid and Greedy search to establish upper bounds on fixed strategies; (3) \textit{Workflow/Prompt optimization frameworks}, including AutoGen~\cite{wu2024autogen}, DSPy~\cite{khattab2023dspy}, GEPA~\cite{agrawal2025gepa}, AFlow~\cite{zhang2025aflow}, and LLM as A Policy (LAP); and (4) \textit{RL baselines}, which treat configuration as either a bandit problem (RL Bandits) or a sequential decision process (RL Episodes) optimized via PPO.

\noindent\textbf{Implementation.} We provide a modular framework supporting arbitrary tool registries, custom workflows, and $n$-agent topologies. Our default experimental setup uses 9 workflow patterns (details in Appendix~\ref{app:workflows}), 4 tools, and 3 agents. We evaluate both Qwen 2.5 7B Instruct~\cite{qwen2.5} and Gemini 2.5 Flash Lite~\cite{comanici2025gemini}, focusing on Qwen in the main text. Full training details and additional results are included in Appendix~\ref{app:training}.

\begin{table}[t]
\centering
\caption{\textbf{Performance comparison across reasoning and tool-use benchmarks.} We compare our approach (ARC and ARC w/o SFT) against base models with tools, search-based methods (Grid/Greedy Search), optimization frameworks (AutoGen, DSPy, GEPA, LLM as a policy (LAP)), and flat PPO policy baselines (RL Bandits, RL Episodes). Our method achieves the best results on most tasks. \textbf{Bold} and \underline{underline} indicates best and second best performance over each task.}
\label{tab:hrl_results}
\vspace{0.4em}

\scriptsize
\setlength{\tabcolsep}{3.2pt}
\renewcommand{\arraystretch}{0.95}

\resizebox{\textwidth}{!}{%
\begin{tabular}{ll | ccc | c | cc | c}
\toprule
& 
& \multicolumn{4}{c}{\textbf{Reasoning Capability} $(\uparrow)$} 
& \multicolumn{3}{c}{\textbf{Tool-Use Capability} $(\uparrow)$} \\
\cmidrule(lr){3-6} \cmidrule(lr){7-9}

\textbf{Method} & \textbf{Model}
& \textbf{GSM8k} & \textbf{DROP} & \textbf{MedQA}
& \textbf{Avg / $\Delta$}
& \textbf{HotpotQA} & \textbf{GAIA}
& \textbf{Avg / $\Delta$} \\
\midrule

\multirow{2}{*}{\makecell{Base Model\\ w/ Tools *}} 
& Qwen
& 37.8\% & 36.4\% & 49.1\% & 41.1\%
& 10.1\% & 2.0\% & 6.1\% \\
& Gemini 
& 62.3\% & 45.9\% & 10.0\% & 39.4\%
& 18.8\% & \underline{3.0\%} & 10.9\% \\
\midrule

Grid Search 
& Qwen    
& 74.0\% & 54.3\% & 53.1\% & 60.5\% ($+19.4$) 
& 27.9\% & 1.0\% & 14.45\% ($+8.35$) \\

Greedy Search 
& Qwen   
& 78.2\% & 57.5\% & 54.9\% & 63.5\% ($+22.4$)
& 28.6\% & 3.0\% & 15.8\% ($+9.7$) \\
\midrule

AutoGen 
& Qwen   
& 74.8\% & 58.0\% & \underline{70.5\%} & 67.8\% ($+26.7$)
& \textbf{34.1\%} & 1.0\% & 17.55\% ($+11.45$) \\

DsPy 
& Qwen   
& 80.4\% & 57.9\% & 57.9\% & 65.4\% ($+24.3$) 
& 28.2\% & 3.0\% & 15.6\% ($+9.5$) \\

GEPA 
& Qwen   
& 83.6\% & 39.3\% & \textbf{87.1\%} & \underline{70.0\% ($+28.9$)}
& 27.4\% & 4.0\% & 15.7\% ($+9.6$) \\

Aflow 
& Qwen    
& 86.5\% & 56.2\% & 60.3\% & 67.7\% ($+26.3$)
& 30.3\% & 1.0\% & 15.65\% ($+9.55$) \\

LAP 
& Qwen    
& 38.3\% & 38.4\% & 49.0\% & 41.9\% ($+0.8$)
& 10.2\% & 0.0\% & 5.1\% ($-1.0$) \\

LAP - FS 
& Qwen  
& 46.1\% & 44.9\% & 49.4\% & 46.8\% ($+5.7$)
& 10.5\% & 1.0\% & 5.75\% ($-0.35$) \\
\midrule

\multirow{2}{*}{RL Bandits} 
& Qwen   
& 80.0\% & 54.3\% & 56.9\% & 63.7\% ($+22.6$)
& 27.3\% & 2.0\% & 14.65\% ($+8.55$) \\
& Gemini 
& 74.7\% & 54.5\% & 58.8\% & 62.7\% ($+23.3$)
& 33.4\% & \underline{3.0\%} & 18.2\% ($+7.3$) \\[0.5ex]

\multirow{2}{*}{RL Episodes}  
& Qwen  
& 85.2\% & 54.3\% & 57.3\% & 65.6\% ($+24.5$)
& 27.8\% & 2.0\% & 14.9\% ($+8.8$) \\
& Gemini 
& 71.0\% & 59.0\% & 60.9\% & 63.6\% ($+24.2$)
& 30.5\% & 1.0\% & 15.75\% ($+4.85$) \\
\midrule

\multirow{2}{*}{ARC w/o SFT} 
& Qwen  
& \underline{87.6\%} & \underline{62.3\%} & 58.4\% & 69.4\% ($+28.3$)
& \underline{33.7\%} & \underline{5.0\%} & \underline{19.35\% ($+13.25$)} \\
& Gemini 
& \underline{85.5\%} & \underline{61.3\%} & \underline{62.8\%} & \underline{69.9\% ($+30.4$)}
& \underline{33.8\%} & \underline{3.0\%} & \underline{18.4\% ($+7.5$)} \\[0.5ex]

\rowcolor{lightblue} 
& Qwen   
& \textbf{88.6\%} & \textbf{63.9\%} & 64.6\% & \textbf{72.4\% ($+31.3$)}
& \textbf{34.1\%} & \textbf{6.0\%} & \textbf{20.05\% ($+13.95$)} \\
\rowcolor{lightblue}
\multirow{-2}{*}{ARC}
& Gemini 
& \textbf{88.5\%} & \textbf{65.6\%} & \textbf{64.7\%} & \textbf{72.9\% ($+33.5$)}
& \textbf{35.7\%} & \textbf{4.0\%} & \textbf{19.85\% ($+8.95$)} \\

\bottomrule
\end{tabular}
}

\vspace{0.2em}
{\scriptsize \textit{*Base model constrained to match the maximum per-query budget available to learned policies.}}

\vspace{-1.61em}
\end{table}

\vspace{-1em}
\subsection{RQ1: Does Learning Configuration Improve Performance?}
\vspace{-0.5em}
Table~\ref{tab:hrl_results} and~\ref{tab:taubench_airline} summarizes the performance across the reasoning, tool-use, and agentic benchmarks. We observe consistent improvements with ARC over baseline methods.


\noindent\textbf{Reasoning Tasks.} 
On GSM8K, ARC achieves $88.6\%$ accuracy with Qwen, outperforming GEPA ($83.6\%$) and RL Episodes ($85.2\%$). On DROP, ARC reaches $63.9\%$ with Qwen and $65.6\%$ with Gemini, improving over the base models by $27.5$ and $19.7$\%, respectively. On MedQA, GEPA performs best ($87.1\%$), while ARC reaches $58.4\%$, still improving over the base model by $15.5$ points. We attribute this gap to \emph{prompt content}; GEPA uses a domain-specific system prompt with 
\begin{wraptable}{r}{0.48\textwidth}
\centering
\caption{\textbf{Performance on $\tau$-Bench (Airline).} We evaluate using Qwen 2.5. Each result reports Pass\textasciicircum k, where larger $k$ measures whether the agent succeeds consistently across trials runs.}
\label{tab:taubench_airline}

\newcommand{\taubenchscale}{0.9}

\scalebox{\taubenchscale}{%
\begin{tabular}{l | ccc}
\toprule
& \multicolumn{3}{c}{\textbf{Airline Task} $(\uparrow)$} \\
\cmidrule(lr){2-4}

\textbf{Method}
& \textbf{Pass\textasciicircum1}
& \textbf{Pass\textasciicircum2}
& \textbf{Pass\textasciicircum3} \\
\midrule

Base Model *
& 9.0\% & 4.5\% & 1.5\% \\

DSPy
& 3.0\% & 0.5\% & 0.0\% \\

ARC w/o SFT
& \underline{16.0\%} & \underline{7.5\%} & \underline{3.0\%} \\

\rowcolor{lightblue}
ARC
& \textbf{18.0\%} & \textbf{11.0\%} & \textbf{6.5\%} \\

\bottomrule
\end{tabular}
}

\vspace{0.2em}
{\scriptsize \textit{*Base model constrained to match the ARC's per-query budget.}}

\vspace{-1.0em}
\end{wraptable}
${\sim}1{,}100$ tokens of medical reasoning heuristics, whereas our prompt library uses general-purpose fragments each averaging ${\sim}50$ tokens. This suggests that in knowledge-intensive domains, prompt semantics can dominate structural choices. Nevertheless, ARC still outperforms RL and search baselines, supporting the value of hierarchical decision-making; combining structural optimization with domain-specific prompts is a promising direction.

\noindent\textbf{Tool-Use and Agentic Tasks.} 
On HotpotQA, ARC achieves $34.1\%$ with Qwen, matching AutoGen ($34.1\%$) and outperforming the strongest RL baseline by $6.3$\%. On GAIA, ARC reaches $6.0\%$ with Qwen and $4.0\%$ with Gemini, improving over the base models by $4.0$ and $1.0$ points, respectively. Notably on $\tau$-Bench (Airline), ARC doubles base-model Pass\textasciicircum1 performance ($9.0\%\!\rightarrow\!18.0\%$), while DSPy reduces performance, likely because fixed prompt instructions hinder adaptation in interactive tasks. Overall, these gains show that learned configurations can better allocate tools and budgets for complex retrieval and agentic tasks.

\noindent\textbf{Model Agnostic.} As shown in Table~\ref{tab:hrl_results}, we observe consistent gains across Qwen 2.5 (open-weights) and Gemini 2.5 (proprietary) validates the generalizability of our framework.

\begin{figure}[t]
\centering


\begin{minipage}[t]{0.49\textwidth}
\vspace{0pt}
\centering
\captionof{table}{\textbf{Workflow diversity across datasets.} We report the number of unique workflows (UW), entropy, and Gini coeff. for different configuration strategies. Higher entropy and lower Gini indicate more diverse exploration of the workflows.}
\label{tab:workflow_diversity}

\footnotesize
\setlength{\tabcolsep}{3.5pt}
\resizebox{\linewidth}{!}{%
\begin{tabular}{l l | c c c}
\toprule
\textbf{Dataset} & \textbf{Method} & \textbf{UW $(\uparrow)$} & \textbf{Entropy $(\uparrow)$} & \textbf{Gini} \\
\midrule

\multirow{3}{*}{HotpotQA}
& Grid Search   & 5 & 0.48 & 0.74 \\
& Greedy Search & 2 & 0.24 & 0.46 \\
& \cellcolor{lightblue} ARC (w/o SFT)
& \cellcolor{lightblue} 9
& \cellcolor{lightblue} 2.13
& \cellcolor{lightblue} 0.58 \\
\midrule

\multirow{3}{*}{GSM8k}
& Grid Search   & 3 & 0.98 & 0.45 \\
& Greedy Search & 2 & 0.24 & 0.46 \\
& \cellcolor{lightblue} ARC (w/o SFT)
& \cellcolor{lightblue} 9
& \cellcolor{lightblue} 1.90
& \cellcolor{lightblue} 0.67 \\
\midrule

\multirow{3}{*}{MedQA}
& Grid Search   & 3 & 1.00 & 0.44 \\
& Greedy Search & 2 & 0.24 & 0.46 \\
& \cellcolor{lightblue} ARC (w/o SFT)
& \cellcolor{lightblue} 9
& \cellcolor{lightblue} 2.95
& \cellcolor{lightblue} 0.30 \\
\bottomrule
\end{tabular}
}
\end{minipage}
\hfill
\begin{minipage}[t]{0.49\textwidth}
\vspace{0pt}
\centering
\captionof{table}{\textbf{Policy transfer across datasets for Qwen.} \textbf{$S_T \to S_N$}: accuracy on training dataset transferring to zero-shot accuracy on new dataset; \textbf{$D_{sim}$}: dataset similarity measured via embedding cosine distance. Policies show moderate transfer to reasoning tasks, with performance degradation of $6-7\%$, while transfer for tool-use tasks is more weaker.}
\label{tab:in_domain_transfer}

\footnotesize
\setlength{\tabcolsep}{4.5pt}
\resizebox{\linewidth}{!}{%
\begin{tabular}{l | l | c | c}
\toprule
\textbf{Capability} 
& \textbf{Train $\rightarrow$ New} 
& \textbf{$S_T \rightarrow S_N$} 
& \textbf{$D_{sim}$} \\
\midrule
\multirow{3}{*}{Reasoning} 
    & GSM8k $\rightarrow$ DROP       & 88.6 $\rightarrow$ 63.0 & 0.79 \\
    & GSM8k $\rightarrow$ MedQA      & 88.6 $\rightarrow$ 57.0 & 0.81 \\
    & DROP $\rightarrow$ GSM8k       & 63.9 $\rightarrow$ 76.0 & 0.79 \\
\midrule
\multirow{3}{*}{\makecell[l]{Tool-Use}} 
    & HotpotQA $\rightarrow$ GAIA    & 34.1 $\rightarrow$ 2.0  & 0.93 \\
    & GAIA $\rightarrow$ HotpotQA    & 6.0 $\rightarrow$ 19.0  & 0.93 \\
    & HotpotQA $\rightarrow$ MedQA   & 34.1 $\rightarrow$ 57.0 & 0.76 \\
\bottomrule
\end{tabular}
}
\end{minipage}

\vspace{1.0em}


\begin{minipage}[t]{0.49\textwidth}
\vspace{0pt}
\centering
\includegraphics[width=\linewidth]{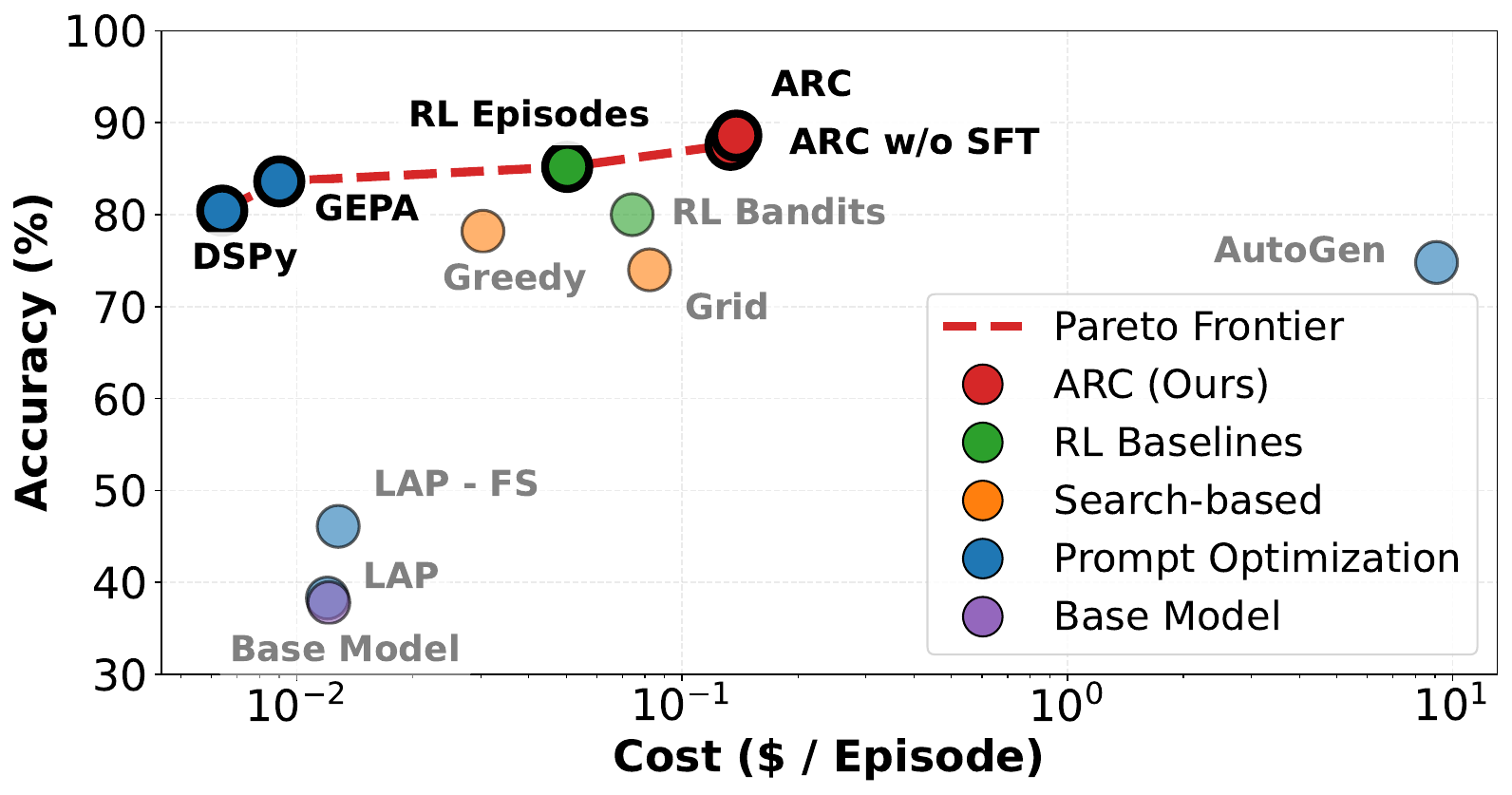}
\captionof{figure}{\textbf{Accuracy Vs. Cost trade-off on GSM8K.} Each point shows average accuracy vs. inference cost. The dashed line denotes the Pareto frontier, representing methods that achieve the best possible accuracy for a given cost. 
}
\label{fig:token_accuracy}
\end{minipage}
\hfill
\begin{minipage}[t]{0.49\textwidth}
\vspace{0pt}
\centering
\includegraphics[width=\linewidth]{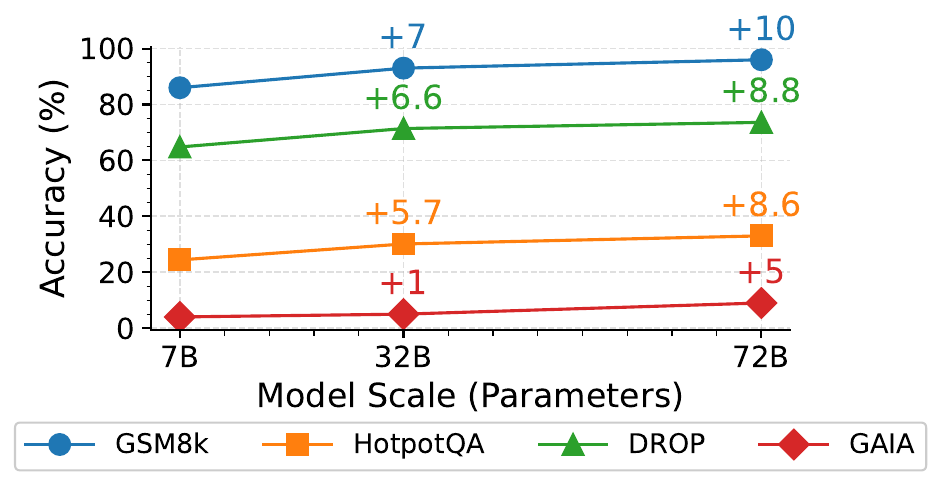}
\captionof{figure}{\textbf{Scaling trends of model accuracy with capacity.} Accuracy as a function of model size for the Qwen 2.5 family (7B, 32B, 72B) across four benchmarks. Performance improves consistently with scale, with gains varying by task complexity.}
\label{fig:scaling_trends}
\end{minipage}

\vspace{-1.2em}
\end{figure}

\subsection{RQ2: Does Adaptive Allocation Improve Efficiency?}
\label{sec:efficiency}

A key advantage of ARC is \emph{resource allocation}: our policy can deploy expensive workflows (e.g., multi-agent verification, iterative refinement) when necessary, while using lightweight strategies for simpler queries. We quantify this in terms of token consumption and workflow complexity.

Figure~\ref{fig:token_accuracy} shows the accuracy–cost trade-off across methods. Cost is computed as token-based API cost per episode using OpenRouter rates for Qwen2.5 7B Instruct. The dashed curve denotes the Pareto~frontier, computed via pairwise dominance. Fixed and search-based baselines lie along this frontier, reflecting the expected trade-off between accuracy and token expense. In contrast, ARC occupies Pareto-optimal regions that strictly improve upon prior methods, achieving high accuracy at lower cost. This indicates that instance-specific adaptation yields more efficient accuracy–cost trade-offs than uniform resource allocation or naive search strategies. Consistent with this observation, Table~\ref{tab:workflow_diversity} shows that ARC explores a diverse set of workflows.

\vspace{-0.5em}
\subsection{RQ3: Can Policies Transfer Across Tasks and Model Capacity?}
\vspace{-0.5em}
We investigate the generalization capabilities of our learned configurations across two dimensions: distinct task domains (to test structural adaptability) and varying model sizes (to test scalability).

\noindent\textbf{Task Specificity Shapes Cross-Domain Transfer.}
\label{sec:transfer}
Table~\ref{tab:in_domain_transfer} reports zero-shot transfer of policies trained on a source dataset ($S_T$) and evaluated on a target subset dataset of ($S_N$).
For reasoning tasks, transfer preserves most in-domain performance: ARC trained on GSM8k achieves $63.0\%$ on DROP (vs.\ $63.9\%$ in-domain) and $57.0\%$ on MedQA (vs.\ $64.7\%$), indicating only moderate degradation.
In contrast, tool-use transfer depends strongly on tool overlap.
ARC trained on HotpotQA transfer reasonably to MedQA ($57.0\%$ vs.\ $64.7\%$), where both rely on web retrieval, but perform poorly on GAIA ($2.0\%$ vs.\ $6.0\%$), which requires fundamentally different multimodal tools.
This suggests that policy transfer is governed more by workflow and tool structure than by semantic similarity alone.


\noindent\textbf{Scaling with Model Capacity.}
\label{sec:scaling}
In contrast to task transfer, we observe zero-shot generalization across model scales. We trained the ARC using the Qwen 2.5 7B backbone and evaluated it directly on the 32B and 72B variants of the same family. As shown in Figure~\ref{fig:scaling_trends}, performance improves with model capacity across tasks (Avg. Kendall’s $\tau = 1.0$ and Pearson’s $r = 0.94$). This indicates that the structural priors learned on smaller models are largely invariant to scale.

\begin{figure}[t]
\centering

\begin{minipage}[t]{0.52\textwidth}
\vspace{0pt}
\centering
\captionof{table}{\textbf{Comparison of training objectives across models.} PPO w/ shaped rewards outperforms GRPO, while SFT provides better generalization than DPO.}
\label{tab:training_objectives}

\footnotesize
\setlength{\tabcolsep}{2.8pt}

\resizebox{\linewidth}{!}{%
\begin{tabular}{l@{\hspace{3pt}}|@{\hspace{3pt}}l l@{\hspace{3pt}}|@{\hspace{3pt}}ccc}

\toprule & \textbf{Method} & \textbf{Model} & \textbf{GSM8k} ($\uparrow$) & \textbf{GAIA} ($\uparrow$)& \textbf{Avg} ($\uparrow$) \\
\midrule

\multirow{4}{*}{\rotatebox{90}{\footnotesize RL}}
  & \multirow{2}{*}{PPO}
    & Qwen  & 87.6\% & 5.0\% & 46.3\%\\
  & & Gemini & 85.5\% & 3.0\% & 44.2\%\\
\cmidrule(lr){2-6}
  & \multirow{2}{*}{GRPO}
    & Qwen  & 82.3\% & 3.0\% & 42.6\%\\
  & & Gemini & 88.9\% & 2.0\% & 45.4\%\\
\midrule

\multirow{4}{*}{\rotatebox{90}{\footnotesize post-train}}
  & \multirow{2}{*}{SFT}
    & Qwen   & 88.6\% & 6.0\% & 47.3\%\\
  & & Gemini & 88.5\% & 4.0\% & 46.2\%\\
\cmidrule(lr){2-6}
  & \multirow{2}{*}{DPO}
    & Qwen & 85.7\% & 4.0\% & 44.85\%\\
  & & Gemini & 80.1\% & 2.0\% & 41.1\%\\
\bottomrule

\end{tabular}
}
\end{minipage}
\hfill
\begin{minipage}[t]{0.47\textwidth}
\vspace{0pt}
\centering
\includegraphics[width=\linewidth]{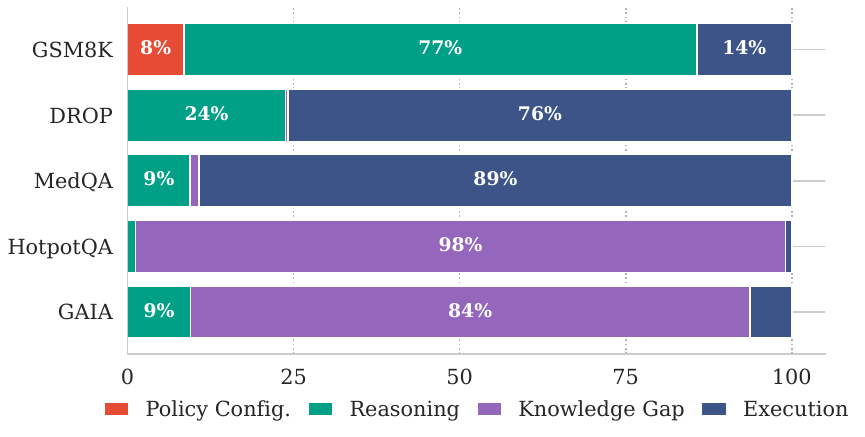}
\captionof{figure}{\textbf{Error distribution.} Reasoning tasks exhibit primarily reasoning errors, while tool-use tasks are dominated by knowledge gap errors. Policy configuration errors remain minimal ($<$10\%) across all datasets.}
\label{fig:error_analysis}
\end{minipage}

\vspace{-1em}
\end{figure}

\vspace{-0.5em}
\subsection{Ablation Studies}
\label{sec:ablations}
\vspace{-0.5em}



\textbf{Why SFT refinement?} Comparing ARC and ARC w/o SFT rows in Table~\ref{tab:hrl_results} and~\ref{tab:taubench_airline}, the SFT phase consistently yields $1$--$3\%$ accuracy gains. Beyond accuracy, \textbf{SFT improves average episode reward by} $\mathbf{\approx 5\text{--}35\%}$ across all datasets and models, validating Theorem~\ref{thm:concentration}: distillation to elite trajectories raises the performance floor, this refinement distills successful strategies from the RL buffer.

\textbf{Alternative training objectives.} We compared our PPO-based training against GRPO. PPO with shaped rewards outperformed GRPO: GRPO struggled with the sparse reward signal from correctness alone. We did not compare against value-based methods due to the high-dimensional action space and the need for efficient exploration. For post-training refinement, we evaluated DPO (direct preference optimization) as an alternative to SFT. Although pairwise comparisons can be constructed from the RL buffer for DPO, this approach exhibits overfitting on the training set, leading to degraded generalization performance compared to SFT. Details are in Table~\ref{tab:training_objectives} and Appendix~\ref{app:training_alternatives}.

\textbf{Search-based alternatives.} Table~\ref{tab:hrl_results} includes grid and greedy search baselines. Grid search achieves $74.0\%$, substantially below ARC ($88.5\%$). Greedy search performs even worse ($36.0\%$), confirming that naive exploration of the configuration space is insufficient and shows the necessity of learning-based approaches for tractably navigating the combinatorial design space.


\vspace{-0.5em}
\subsection{Error Analysis}
\vspace{-0.5em}
\label{sec:errors}

We categorize errors across five benchmarks (reasoning and tool-use) into four types: (1) \textit{policy configuration errors}, where ARC selects suboptimal workflows or tools; (2) \textit{reasoning errors}, where the LLM applies incorrect logic; (3) \textit{knowledge gap}, where the LLM fails to retrieve correct information; and (4) \textit{execution errors}, including arithmetic mistakes and extraction failures. (see Appendix~\ref{app:error_methodology} for detection criteria)

Figure~\ref{fig:error_analysis} reveals distinct error profiles of ARC. On reasoning tasks (GSM8k), 77\% of errors are reasoning failures with only 9\% from policy misconfiguration. On tool-use tasks (HotpotQA, GAIA), knowledge gap errors dominate (84-98\%). Critically, policy configuration errors remain below 10\% across all benchmarks, indicating the learned policy effectively adapts to query requirements.

%% file: sections/conclusion.tex
\vspace{-0.5em}
\section{Conclusion}
\vspace{-0.5em}
We introduced ARC, a hierarchical RL framework for query-adaptive configuration of LLM-based agents. ARC formulates agent configuration as an SMDP, where workflows, tools, budgets, and prompts define temporally extended options, and decomposes this large combinatorial space into a tractable two-level policy. By combining masked reinforcement learning with supervised fine-tuning on high-reward trajectories, ARC learns to efficiently select query-specific configurations without updating the backbone LLM. Across reasoning, tool-use, and agentic benchmarks, ARC consistently improves over budget-matched tool-augmented LLMs, increasing average reasoning accuracy by $31.3$\%, tool-use accuracy by $13.95$\%, and doubling $\tau$-Bench (Airline) Pass\textasciicircum1 success from $9.0\%$ to $18.0\%$. These results highlight the value of learning hierarchical, query-wise agent configurations for improving both performance and resource allocation. More broadly, ARC shows that adaptive architectural decision-making is a practical direction for building flexible and compute-efficient LLM-based agent systems.



%% file: sections/appendix.tex
\newpage
\appendix
\onecolumn

\section{Agentic Workflows}
\label{app:workflows}
\begin{figure}[h]
\centering
\includegraphics[width=\columnwidth]{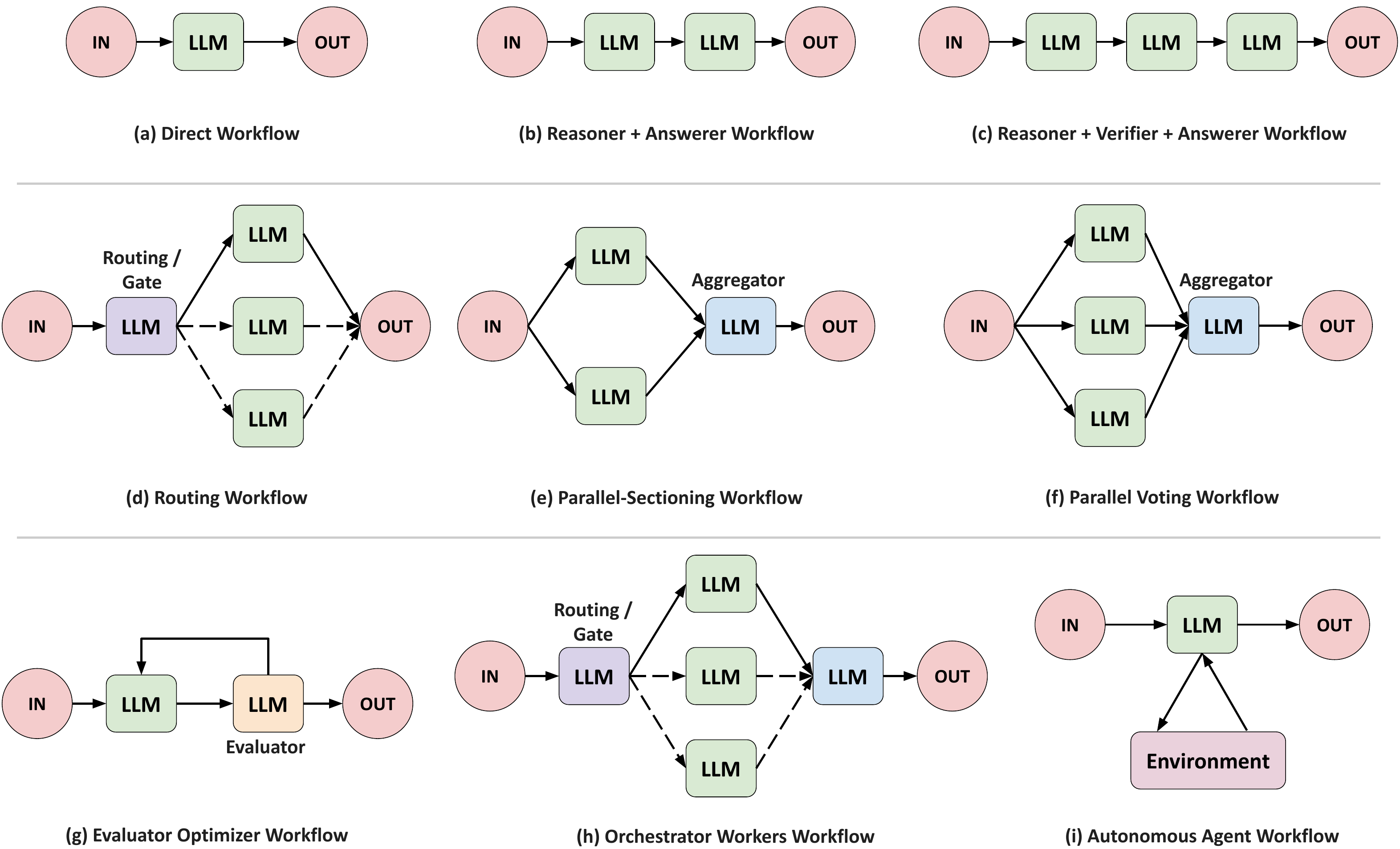}
    \caption{Overview of the nine agentic workflows: Direct (0), Reason+Ans (1), Reason+Verify+Ans (2), Routing (3), Parallel-Sectioning (4), Parallel-Voting (5), Orchestrator-Workers (6), Evaluator-Optimizer (7), and Autonomous-Agent (8). Each workflow defines a distinct pattern of LLM calls and agent interactions.}
    \label{fig:workflows}
\end{figure}

Our framework supports nine agentic workflows, ranging from single-call baselines to multi-agent orchestration patterns. Each workflow specifies a computation graph over LLM calls, with configurable tool access and token budgets per agent role. Below we summarize all workflows concisely.

\textbf{Workflow 0: Direct.} A single LLM call produces the final answer directly. Tools may optionally be used.  
\textbf{LLM Calls:} 1, \textbf{Agent2 Tools:} Not used.

\textbf{Workflow 1: Reason + Answer.} The Reasoner generates intermediate reasoning, which the Answerer converts into a final response.  
\textbf{LLM Calls:} 2, \textbf{Agent2 Tools:} Not used.

\textbf{Workflow 2: Reason + Verify + Answer.}An explicit verification step critiques the reasoning before answer synthesis.  
\textbf{LLM Calls:} 3, \textbf{Agent2 Tools:} Verifier.

\textbf{Workflow 3: Routing.}A router dispatches the query to one of two specialized reasoners with different tool configurations.  
\textbf{LLM Calls:} 3, \textbf{Agent2 Tools:} Conditional (Reasoner2).

\textbf{Workflow 4: Parallel Sectioning.}The question is decomposed into independent subtasks solved in parallel and then aggregated.  
\textbf{LLM Calls:} 4, \textbf{Agent2 Tools:} Worker1.

\textbf{Workflow 5: Parallel Voting.}Multiple independent attempts are aggregated via majority voting.  
\textbf{LLM Calls:} 4, \textbf{Agent2 Tools:} Not used.

\textbf{Workflow 6: Orchestrator-Workers.}The Orchestrator delegates subtasks to workers, separating planning from execution.  
\textbf{LLM Calls:} 4, \textbf{Agent2 Tools:} Workers.

\textbf{Workflow 7: Evaluator-Optimizer.}An iterative generate--evaluate--refine loop runs until convergence or a maximum number of iterations.  
\textbf{LLM Calls:} 4-7, \textbf{Agent2 Tools:} Evaluator.

\textbf{Workflow 8: Autonomous Agent.}The agent iteratively reasons and invokes tools without an explicit evaluator.  
\textbf{LLM Calls:} 4, \textbf{Agent2 Tools:} Iterations 2+.

\section{Computational Resources}
\label{app:resources}

This section summarizes the computational resources required to reproduce our experiments (training, evaluation, and plotting). We support two execution modes: \textit{local} (running an open-weight LLM on your own GPU) and \textit{API} (calling an external provider such as OpenRouter\footnote{\url{https://openrouter.ai/}}). The codebase works on a Python~3.10+ installation and a CUDA-capable machine when using local models.

For local (HuggingFace) inference and training, it is sufficient to have a single modern GPU for 7B-class models such as Qwen~2.5~7B-Instruct, together with at least 64\,GB of CPU RAM to store datasets and logs. Models in the 1.5B--3B range can typically be trained and evaluated on commodity GPUs with 24\,GB of VRAM, which is adequate for ablations and scaling experiments. For larger models (14B parameters and above), practical training and evaluation usually require either multiple GPUs with $\geq$40\,GB VRAM each or aggressive quantization; in such cases, it is often preferable to access hosted checkpoints via an API. Under these conditions, reproducing the GSM8K and DROP results with a 7B-class model is feasible on a single A100/V100/4090-class GPU with 40\,GB VRAM (or 24\,GB with 4-bit quantization) and 64\,GB CPU RAM.

In API mode no accelerator is needed; the limiting factors are latency, concurrency limits, and monetary budget. Our evaluation scripts support parallel workers so that, when the provider permits concurrent requests, hundreds to thousands of evaluation episodes can be processed in a few hours from a single machine.

\section{Proofs of Theoretical Guarantees}
\label{app:proofs}
 
\subsection{Proof of Theorem~\ref{thm:convergence}}
 
\begin{proof}
\textbf{(i) Contraction.} Let $Q_1, Q_2$ be bounded. For any $(s,o)$:
\begin{align}
|(\mathcal{T}Q_1)(s,o) - (\mathcal{T}Q_2)(s,o)|
&= \left|\mathbb{E}\!\left[\gamma^{k_o}\!\left(\max_{o'} Q_1(s',o') - \max_{o'} Q_2(s',o')\right) \middle| s,o\right]\right| \\
&\leq \mathbb{E}\!\left[\gamma^{k_o} \max_{o'} |Q_1(s',o') - Q_2(s',o')| \middle| s,o\right] \\
&\leq \mathbb{E}[\gamma^{k_o} \mid s,o]\, \|Q_1 - Q_2\|_\infty.
\end{align}
Since $k_o \geq 1$, we have $\mathbb{E}[\gamma^{k_o} \mid s,o] \leq \gamma < 1$. Taking the supremum:
\[
\|\mathcal{T}Q_1 - \mathcal{T}Q_2\|_\infty \leq \gamma\, \|Q_1 - Q_2\|_\infty.
\]
By the Banach fixed-point theorem, $\mathcal{T}$ has a unique fixed point $Q^*$.
 
\textbf{(ii) Convergence.} The SMDP Q-learning update is:
\[
Q_{n+1}(s,o) = (1-\alpha_n)\,Q_n(s,o) + \alpha_n\!\left[R_n + \gamma^{k_n}\max_{o'}Q_n(s'_n,o')\right].
\]
Write this as $Q_{n+1}(s,o) = Q_n(s,o) + \alpha_n[(\mathcal{T}Q_n)(s,o) - Q_n(s,o) + w_n]$, where $w_n$ has zero conditional mean. The conditions of Theorem~1 of Jaakkola et al.\ (\cite{jaakkola1993convergence}) hold: (a)~Robbins--Monro step sizes by assumption, (b)~all pairs visited infinitely often by assumption, (c)~$\mathcal{T}$ is a $\gamma$-contraction by Part~(i), (d)~$\mathrm{Var}[w_n \mid Q_n,s,o] \leq C(1+\|Q_n\|_\infty)^2$ since rewards and durations are bounded. Therefore $Q_n(s,o) \to Q^*(s,o)$ a.s.\ for all $(s,o)$.
\end{proof}
 
\subsection{Proof of Theorem~\ref{thm:concentration}}
 
\begin{proof}
Let $\hat{p}_{\text{elite}}(c|s) = n(s,c)/n(s)$ be the empirical distribution over $\mathcal{D}_{\text{elite}}$.
 
\textbf{(i) Support restriction.} The SFT objective $\max_\theta \mathbb{E}_{(s,c^*)\sim\mathcal{D}_{\text{elite}}}[\log\pi_\theta(c^*|s)]$ is equivalent to minimizing $D_{\mathrm{KL}}(\hat{p}_{\text{elite}}\|\pi_\theta) + H(\hat{p}_{\text{elite}})$. Under sufficient capacity, the minimum is attained at $\pi_{arc}^{\text{SFT}} = \hat{p}_{\text{elite}}$. Since $\hat{p}_{\text{elite}}(c|s) = 0$ for $c \notin \mathcal{C}_{\text{elite}}(s)$, we have $\pi_{arc}^{\text{SFT}}(c|s) > 0 \Rightarrow c \in \mathcal{C}_{\text{elite}}(s)$.
 
\textbf{(ii) Reward floor.} Every $(s,c) \in \mathcal{D}_{\text{elite}} = \{(s_i,c_i^*)\in\mathcal{B} : \text{correct}_i \wedge R(s_i,c_i^*)\geq\tau\}$ satisfies $R(s,c) \geq \tau$. Combined with (i):
\[
\mathbb{E}_{c \sim \pi_{arc}^{\text{SFT}}}[R(s,c)] = \sum_{c \in \mathcal{C}_{\text{elite}}(s)} \pi_{arc}^{\text{SFT}}(c|s)\, R(s,c) \geq \tau \sum_{c \in \mathcal{C}_{\text{elite}}(s)} \pi_{arc}^{\text{SFT}}(c|s) = \tau. \qedhere
\]
\end{proof}

\textbf{Theorem~\ref{thm:convergence}} The contraction result establishes that the SMDP has a unique optimal value function $Q^*$, meaning RL has a well-defined target to converge toward. Without this, it would be unclear whether the optimization is even searching for something that exists. The $\gamma^{k_o}$ discounting is central: because each Option consumes at least one LLM call ($k_o \geq 1$), the contraction modulus is at most $\gamma < 1$, ensuring $Q^*$ is unique. Options with longer durations (e.g., Evaluator-Optimizer at $k_o \approx 5$) contribute a tighter contraction factor ($\gamma^5 \approx 0.59$ for $\gamma = 0.9$) than short ones, reflecting that heavier Options are discounted more aggressively.
 
\textbf{Theorem~\ref{thm:concentration}} Support restriction ensures that $\pi_{arc}^{\text{SFT}}$ only proposes configurations that were successful during training, preventing the policy from inventing novel, untested configurations at deployment time. The reward guarantee provides a performance floor: since all elite configurations achieve reward at least $\tau$, and $\pi_{arc}^{\text{SFT}}$ samples exclusively from this set, expected performance is guaranteed to be at least $\tau$. In our experiments, $\tau$ corresponds to the 70th percentile of rewards in the RL buffer, meaning $\pi_{arc}^{\text{SFT}}$ is guaranteed to perform at least as well as the top 30\% of RL trajectories in expectation.

These theoretical properties complement the empirical findings in Section~\ref{sec:ablations}, where we observe that adding the SFT phase consistently improves performance and reduces variance across runs.

\section{Ablation: Identifying the Best Embedding for State Representation}
\label{app:metaclip}

The structure policy $\pi_{\text{struct}}$ conditions on a state representation $s_q = [\phi(q); \mathbf{f}_q]$, where $\phi(q)$ is a learned embedding and $\mathbf{f}_q$ contains hand-crafted features. The choice of embedding model $\phi(\cdot)$ is critical: if the representation fails to capture task-relevant semantics, the policy cannot learn to differentiate queries that require different configurations.

We conduct a comprehensive ablation across 19 embedding models spanning text-only encoders (Sentence-T5~\cite{ni2022sentence}, E5~\cite{wang2022text}, MiniLM, MPNet~\cite{song2020mpnet}), vision-language models (CLIP, MetaCLIP, SigLIP~\cite{zhai2023sigmoid}, FLAVA~\cite{singh2022flava}), and hybrid approaches (Jina-CLIP~\cite{jina}). For each model, we evaluate embeddings in two modes: \textit{native} (using the model's original dimensionality) and \textit{projected} (linearly projecting to a fixed 768-dimensional space for fair comparison). We design four evaluation tasks that directly measure properties relevant to RL policy learning:

\begin{wrapfigure}{r}{0.5\textwidth}
  \centering
  \vspace{-10pt}
  \includegraphics[width=0.5\textwidth]{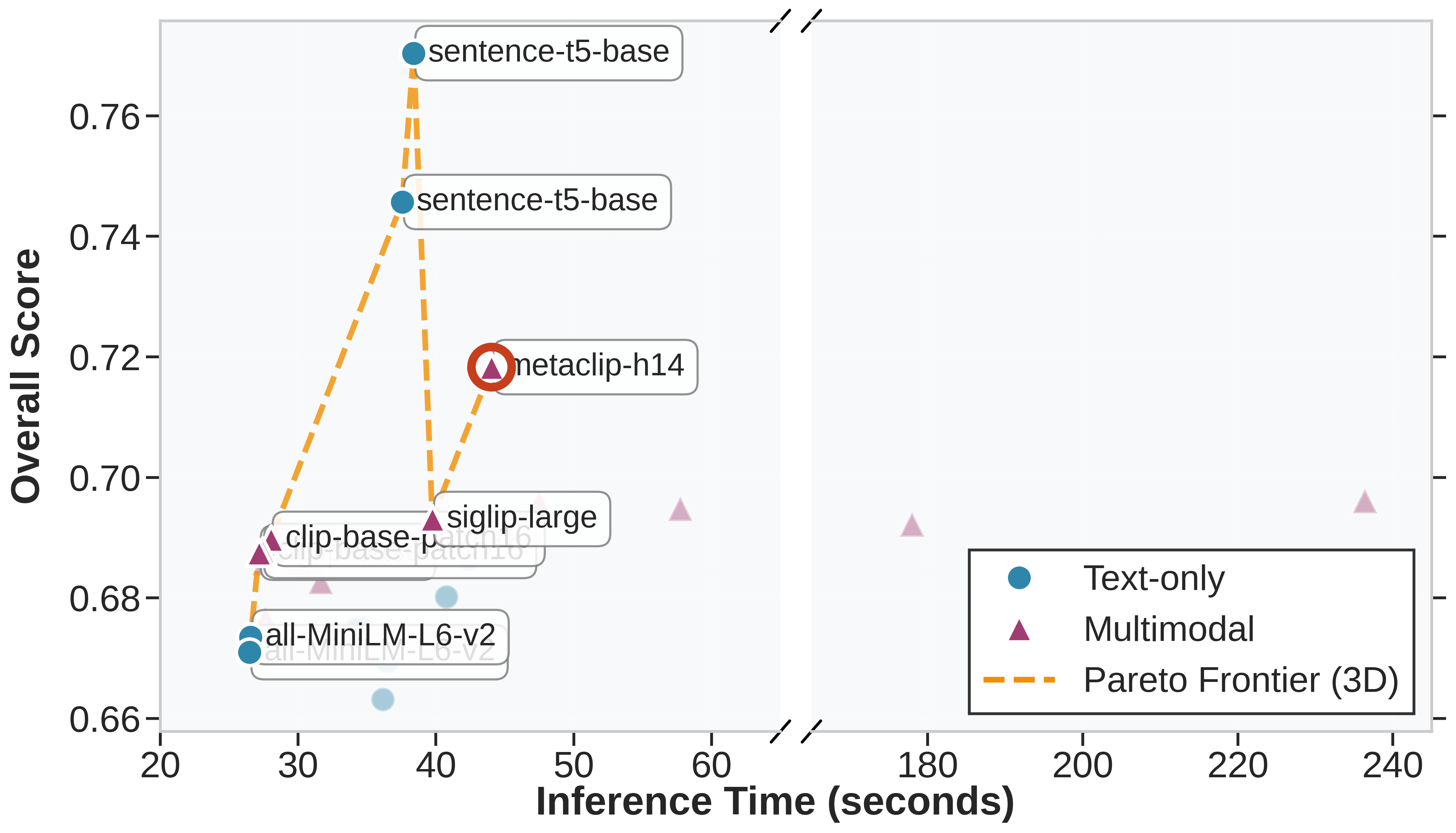}
  \vspace{-8pt}
  \caption{Overall score vs.\ runtime vs.\ multimodality with the Pareto frontier under a three-dimensional dominance criterion (runtime $\downarrow$, score $\uparrow$, multimodality $\uparrow$). Circles denote text-only models and triangles denote multimodal models. The selected multimodal model (MetaCLIP-H14) is highlighted.}
  \label{fig:pareto_metaclip}
  \vspace{-10pt}
\end{wrapfigure}

\textbf{Clustering Quality (ARI).} We evaluate how well embeddings group semantically similar questions using Adjusted Rand Index (ARI). RL policies rely on embeddings that place similar problems close together so that similar inputs lead to consistent decisions. We cluster embeddings using $k$-means and compare against ground-truth labels (dataset $\times$ tool type). ARI ranges from $-1$ to $1$, where $1.0$ indicates perfect alignment with ground truth.

\textbf{Classification Accuracy (Cls. Acc).} We test whether embeddings distinguish problem types and tool requirements, which is essential for the policy to select correct workflows and tools. We train linear classifiers for two tasks: (1) \textit{dataset classification} predicting the source benchmark (GSM8k, HotpotQA, AIME, MedQA), and (2) \textit{tool requirement classification} predicting the required tool (calculator, web\_search, python, none). The reported score is the average cross-validation accuracy across both tasks.

\textbf{Complexity Ranking (Complexity).} We assess whether embeddings capture problem difficulty, which is crucial for budget allocation decisions. We define a complexity score using heuristics (question length, answer length, difficulty-indicating keywords) and train a Ridge regressor to predict complexity from embeddings. We report the Spearman correlation between predicted and true complexity rankings, which measures monotonic relationship quality. Higher values indicate better difficulty ordering.

\textbf{Decision Prediction (Decision).} This task most directly simulates the RL objective. We train a multi-head MLP to simultaneously predict: (1) workflow ID, (2) required tool, and (3) compute budget tier (Low/Mid/High). The Decision Score is the average accuracy across all three prediction heads, weighted by a combined accuracy metric that requires all three predictions to be correct simultaneously.

\textbf{Overall Score.} We compute a weighted average across all four metrics: $\text{Overall} = 0.15 \cdot \text{ARI} + 0.25 \cdot \text{Cls.~Acc} + 0.25 \cdot \text{Complexity} + 0.35 \cdot \text{Decision}$
The weights reflect the relative importance of each property for downstream RL performance, with Decision Prediction receiving the highest weight as it most closely matches the policy's objective.

\newcolumntype{Y}{>{\centering\arraybackslash}X}

\begin{table*}[h]
\centering
\caption{
Evaluation of embedding models for downstream RL tasks. 
The highlighted \textit{sentence-t5-base} (text-only) and \textit{MetaCLIP-H14} (multimodal) are the final candidates, with their performance–efficiency trade-off shown as a Pareto frontier in Fig.~\ref{fig:pareto_metaclip}.
}
\tiny
\setlength{\tabcolsep}{3pt} 
\begin{tabularx}{\textwidth}{l| l| Y Y Y Y |Y| Y}
\toprule
\textbf{Embedder} & \textbf{Mode} &
\textbf{ARI} &
\textbf{Cls. Acc} &
\textbf{Complexity} &
\textbf{Decision} &
\textbf{Overall} &
\textbf{Time (s)} \\
\midrule

\cellcolor{lightblue}  sentence-t5-base (768D) & \cellcolor{lightblue} native
& \cellcolor{lightblue} 0.5603 $\pm$ 0.1019 & \cellcolor{lightblue} 0.8733 $\pm$ 0.0048 & \cellcolor{lightblue} 0.9261 $\pm$ 0.0099 & \cellcolor{lightblue} 0.7221 $\pm$ 0.0070 & \cellcolor{lightblue} 0.7704 $\pm$ 0.0305 & \cellcolor{lightblue} 38.39 $\pm$ 1.24 \\ \addlinespace

sentence-t5-base (768D) & projected
& 0.4867 $\pm$ 0.1017 & 0.8635 $\pm$ 0.0084 & 0.9189 $\pm$ 0.0109 & 0.7138 $\pm$ 0.0018 & 0.7457 $\pm$ 0.0229 & 37.56 $\pm$ 0.89 \\ \addlinespace

\cellcolor{lightblue} MetaCLIP-H14 (1024D) & \cellcolor{lightblue} native
& \cellcolor{lightblue} 0.4074 $\pm$ 0.0525 & \cellcolor{lightblue} 0.8514 $\pm$ 0.0043 & \cellcolor{lightblue} 0.8944 $\pm$ 0.0112 & \cellcolor{lightblue} 0.7198 $\pm$ 0.0139 & \cellcolor{lightblue} 0.7183 $\pm$ 0.0197 & \cellcolor{lightblue} 44.04 $\pm$ 1.07 \\ \addlinespace

MetaCLIP-2-Worldwide-L14 (768D) & native
& 0.3470 & 0.8553 & 0.9131 & 0.7451 & 0.7151 & 43.91 \\ \addlinespace

MetaCLIP-2-Worldwide-L14 (768D) & projected
& 0.3272 & 0.8470 & 0.9114 & 0.7245 & 0.7025 & 49.20 \\ \addlinespace

Jina-CLIP-v2 & projected
& 0.3736 $\pm$ 0.0415 & 0.8476 $\pm$ 0.0028 & 0.8833 $\pm$ 0.0124 & 0.6798 $\pm$ 0.0027 & 0.6961 $\pm$ 0.0129 & 236.38 $\pm$ 2.23 \\ \addlinespace

SigLIP-Large & projected
& 0.3790 $\pm$ 0.0026 & 0.8048 $\pm$ 0.0132 & 0.8857 $\pm$ 0.0088 & 0.7135 $\pm$ 0.0030 & 0.6958 $\pm$ 0.0057 & 47.45 $\pm$ 1.58 \\ \addlinespace

SigLIP-Large (1024D) & native
& 0.3469 $\pm$ 0.0412 & 0.8081 $\pm$ 0.0132 & 0.8911 $\pm$ 0.0095 & 0.7264 $\pm$ 0.0037 & 0.6931 $\pm$ 0.0090 & 39.74 $\pm$ 1.19 \\ \addlinespace

CLIP-Base-Patch16 (512D) & native
& 0.3370 $\pm$ 0.0036 & 0.8454 $\pm$ 0.0094 & 0.8894 $\pm$ 0.0141 & 0.6872 $\pm$ 0.0121 & 0.6898 $\pm$ 0.0061 & 28.03 $\pm$ 0.87 \\ \addlinespace

CLIP-Large (768D) & native
& 0.3406 $\pm$ 0.0067 & 0.8416 $\pm$ 0.0082 & 0.8811 $\pm$ 0.0116 & 0.6955 $\pm$ 0.0096 & 0.6897 $\pm$ 0.0050 & 33.74 $\pm$ 5.19 \\ \addlinespace

SigLIP-Base & projected
& 0.3264 $\pm$ 0.0447 & 0.8251 $\pm$ 0.0102 & 0.8797 $\pm$ 0.0131 & 0.7187 $\pm$ 0.0035 & 0.6875 $\pm$ 0.0115 & 27.17 $\pm$ 1.10 \\ \addlinespace

all-mpnet-base-v2 (768D) & native
& 0.3356 $\pm$ 0.0181 & 0.8426 $\pm$ 0.0160 & 0.8894 $\pm$ 0.0076 & 0.6784 $\pm$ 0.0071 & 0.6865 $\pm$ 0.0097 & 42.37 $\pm$ 0.92 \\ \addlinespace

CLIP-Large & projected
& 0.3293 $\pm$ 0.0064 & 0.8366 $\pm$ 0.0113 & 0.8744 $\pm$ 0.0111 & 0.7029 $\pm$ 0.0051 & 0.6858 $\pm$ 0.0064 & 30.08 $\pm$ 0.87 \\ \addlinespace

siglip-base (768D) & native
& 0.3025 $\pm$ 0.0183 & 0.8284 $\pm$ 0.0121 & 0.8863 $\pm$ 0.0087 & 0.7262 $\pm$ 0.0128 & 0.6858 $\pm$ 0.0084 & 27.54 $\pm$ 1.13 \\ \addlinespace

flava-full & projected
& 0.3526 $\pm$ 0.0294 & 0.8476 $\pm$ 0.0086 & 0.8898 $\pm$ 0.0047 & 0.6513 $\pm$ 0.0041 & 0.6853 $\pm$ 0.7577 & 31.11 $\pm$ 0.01 \\ \addlinespace

flava-full (768D) & native
& 0.3362 $\pm$ 0.0086 & 0.8509 $\pm$ 0.0090 & 0.8923 $\pm$ 0.0081 & 0.6513 $\pm$ 0.0051 & 0.6827 $\pm$ 0.0031 & 31.64 $\pm$ 0.97 \\ \addlinespace

all-MiniLM-L12-v2 (384D) & native
& 0.3304 $\pm$ 0.0195 & 0.8432 $\pm$ 0.0034 & 0.8766 $\pm$ 0.0219 & 0.6497 $\pm$ 0.0115 & 0.6750 $\pm$ 0.0118 & 34.55 $\pm$ 1.24 \\ \addlinespace

all-MiniLM-L6-v2 & projected
& 0.3250 $\pm$ 0.0028 & 0.8399 $\pm$ 0.0246 & 0.8624 $\pm$ 0.0113 & 0.6670 $\pm$ 0.0148 & 0.6735 $\pm$ 0.0071 & 26.56 $\pm$ 1.26 \\ \addlinespace

e5-base (768D) & native
& 0.3225 $\pm$ 0.0003 & 0.8432 $\pm$ 0.0054 & 0.8705 $\pm$ 0.0133 & 0.6412 $\pm$ 0.0056 & 0.6694 $\pm$ 0.0042 & 36.48 $\pm$ 1.05 \\

\bottomrule
\end{tabularx}
\label{tab:embedding_ablation_full}
\end{table*}

\section{Ablation: Identifying the Best Prompt Generator}
\label{app:best_prompter}

The prompt policy $\pi_{\text{prompt}}$ selects instructions from a predefined library of prompts, semantic instruction fragments such as ``Decompose the problem'' or ``Verify intermediate steps''. To ensure high-quality instructions, we augment base templates with dataset-specific prompts generated via meta-prompting. This ablation identifies the optimal LLM for generating these prompts.

We generate atoms for 5 datasets (GSM8k, HotpotQA, GAIA, MedQA, AIME) using 12 models spanning OpenAI (GPT-4o, GPT-4 Turbo, GPT-5.2), Anthropic (Claude 3.5 Haiku/Sonnet), Meta (Llama 3.1 8B/70B), Mistral (Large), Google (Gemini 2.5 Pro), and Qwen (2.5 7B/72B). Each model generates atoms for three agent roles: reasoner, verifier, and answerer.

We evaluate generated atoms on two axes:
\begin{itemize}[noitemsep]
    \item \textbf{Diversity (40\% weight)}: Uniqueness (1 $-$ mean pairwise cosine similarity), strategy coverage (ratio of covered reasoning strategies), and semantic diversity (silhouette score from $k$-means clustering).
    \item \textbf{Quality (60\% weight)}: Coherence (GPT-5 rating 1--10), specificity (cosine distance from base atoms), and clarity (GPT-5 rating 1--10).
\end{itemize}
The combined score is $0.4 \times \text{Diversity} + 0.6 \times \text{Quality}$.

Table~\ref{tab:prompt_generator_ablation} shows results aggregated by model. GPT-5.2 achieves the highest combined score (0.487 average) across all datasets, driven by superior quality metrics (coherence 8.79, clarity 5.50). Notably, GPT-5.2 also achieves the highest strategy coverage (0.80), indicating diverse reasoning approaches. Based on these results, we use GPT-5.2-generated atoms in our prompt library.

\begin{table*}[t]
\centering
\tiny
\caption{Ablation study comparing LLM models for generating prompt atoms. Each row is a model-dataset combination, ranked by combined score (40\% diversity + 60\% quality). \textbf{Diversity metrics}: Uniqueness (1 $-$ mean pairwise cosine similarity), Coverage (fraction of 6 strategy types), Semantic (silhouette score from $k$-means). \textbf{Quality metrics}: Coherence and Clarity (GPT-5 rated 1--10), Specificity (distance from base atoms).}
\label{tab:prompt_generator_ablation}
\setlength{\tabcolsep}{10pt}
\begin{tabular}{l c l | ccc | ccc | c}
\toprule
& & & \multicolumn{3}{c|}{\textbf{Diversity}} & \multicolumn{3}{c|}{\textbf{Quality}} & \\
\textbf{Model} & \textbf{Size} & \textbf{Dataset} & \textbf{Unique} & \textbf{Cover} & \textbf{Sem.} & \textbf{Coher.} & \textbf{Spec.} & \textbf{Clar.} & \textbf{Score} \\
\midrule
GPT-5.2              & Closed  & gsm8k     & 0.18 & 0.83 & 0.08 & 8.76 & 0.24 & 5.50 & \textbf{0.48} \\
GPT-5.2              & Closed  & medqa     & 0.24 & 0.67 & 0.19 & 8.94 & 0.26 & 5.50 & \textbf{0.49} \\
GPT-5.2              & Closed  & gaia      & 0.19 & 0.83 & 0.10 & 8.56 & 0.25 & 5.50 & \textbf{0.48} \\
GPT-5.2              & Closed  & hotpotqa  & 0.17 & 0.83 & 0.05 & 8.88 & 0.27 & 5.50 & \textbf{0.48} \\
GPT-5.2              & Closed  & aime25    & 0.15 & 0.83 & 0.18 & 8.81 & 0.23 & 5.50 & \textbf{0.49} \\
\midrule
Llama 3.1 70B        & 70B       & gaia      & 0.18 & 0.83 & 0.29 & 5.50 & 0.21 & 5.50 & 0.44 \\
Llama 3.1 8B         & 8B        & gaia      & 0.24 & 0.67 & 0.30 & 5.53 & 0.22 & 5.50 & 0.42 \\
GPT-4o               & Closed & medqa    & 0.19 & 0.67 & 0.30 & 5.50 & 0.19 & 5.76 & 0.42 \\
Claude 3.5 Haiku     & Closed  & hotpotqa  & 0.17 & 0.67 & 0.25 & 5.50 & 0.25 & 5.50 & 0.41 \\
Claude 3.5 Haiku     & Closed  & aime25    & 0.15 & 0.67 & 0.19 & 5.71 & 0.25 & 5.71 & 0.41 \\
Llama 3.1 8B         & 8B        & hotpotqa  & 0.23 & 0.67 & 0.27 & 5.50 & 0.22 & 5.50 & 0.41 \\
GPT-4 Turbo          & Closed & medqa    & 0.19 & 0.67 & 0.29 & 5.50 & 0.20 & 5.50 & 0.40 \\
Llama 3.1 70B        & 70B       & hotpotqa  & 0.16 & 0.67 & 0.25 & 5.50 & 0.23 & 5.50 & 0.40 \\
Qwen 2.5 7B          & 7B        & gaia      & 0.25 & 0.50 & 0.28 & 5.50 & 0.21 & 5.50 & 0.40 \\
Qwen 2.5 7B          & 7B        & hotpotqa  & 0.22 & 0.67 & 0.25 & 5.50 & 0.23 & 5.50 & 0.40 \\
\bottomrule
\end{tabular}
\end{table*}

\section{Training Details}
\label{app:training}

\textbf{PPO Training.} We use the following hyperparameters:
\begin{itemize}[noitemsep]
    \item Learning rate: $3 \times 10^{-4}$ (structure), $5 \times 10^{-5}$ (prompt)
    \item Batch size: 32 episodes
    \item PPO clip epsilon: $\epsilon = 0.2$
    \item Discount factor: $\gamma = 0.95$
    \item Entropy coefficient: $\beta_{\text{ent}} = 0.05$
    \item Value loss coefficient: $0.5$
    \item Gradient clipping: max norm $0.5$
    \item PPO epochs per batch: $E = 4$
    \item Total training episodes: 4,000 per dataset minimum
\end{itemize}

\textbf{Reward shaping coefficients:}
\begin{itemize}[noitemsep]
    \item Task success weight: $\alpha = 5.0$
    \item Step penalty: $\beta_s = 0.02$
    \item Token penalty: $\beta_t = 0.03$
    \item Tool shaping: $\delta_1 = 0.1$, $\delta_2 = 0.2$, $\delta_3 = 0.3$
\end{itemize}

\textbf{SFT Post-training:}
\begin{itemize}[noitemsep]
    \item Structure LR: $1 \times 10^{-4}$, Prompt LR: $5 \times 10^{-6}$
    \item Entropy regularization: $0.01$
    \item Reward threshold $\tau$: 4.0 (top 30\% of episodes)
    \item Epochs: 3
\end{itemize}

For the SFT refinement phase:
\begin{itemize}[noitemsep]
    \item Learning rate: $1 \times 10^{-4}$
    \item Elite threshold: $\tau$ set to retain top 30\% by reward
    \item Number of SFT epochs: $E_{\text{SFT}} = 10$
    \item Batch size: 32
\end{itemize}

\begin{figure}[h]
    \centering
    \includegraphics[width=0.32\linewidth]{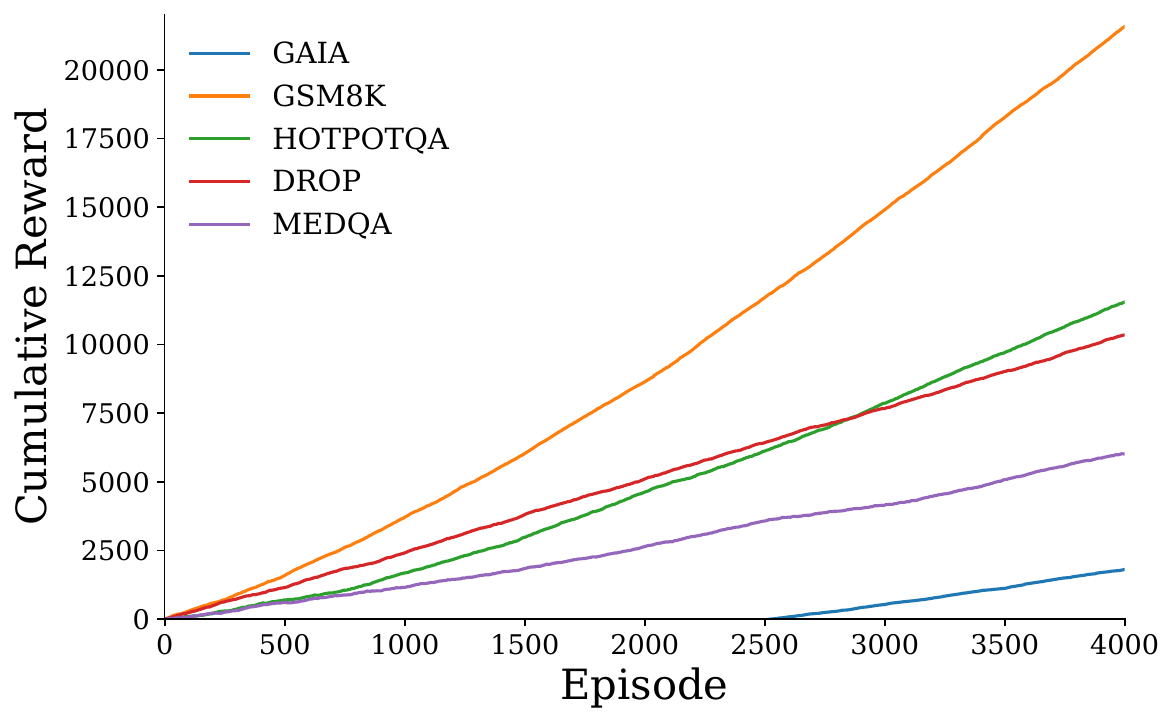}
    \includegraphics[width=0.32\linewidth]{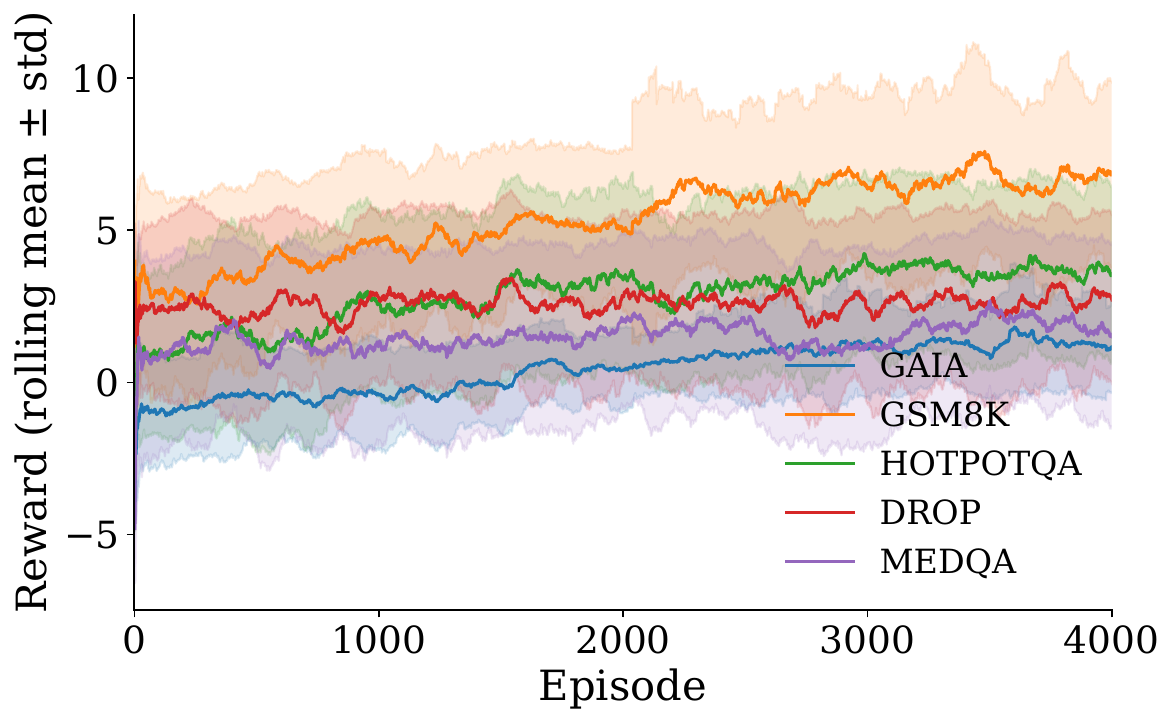}
    \includegraphics[width=0.32\linewidth]{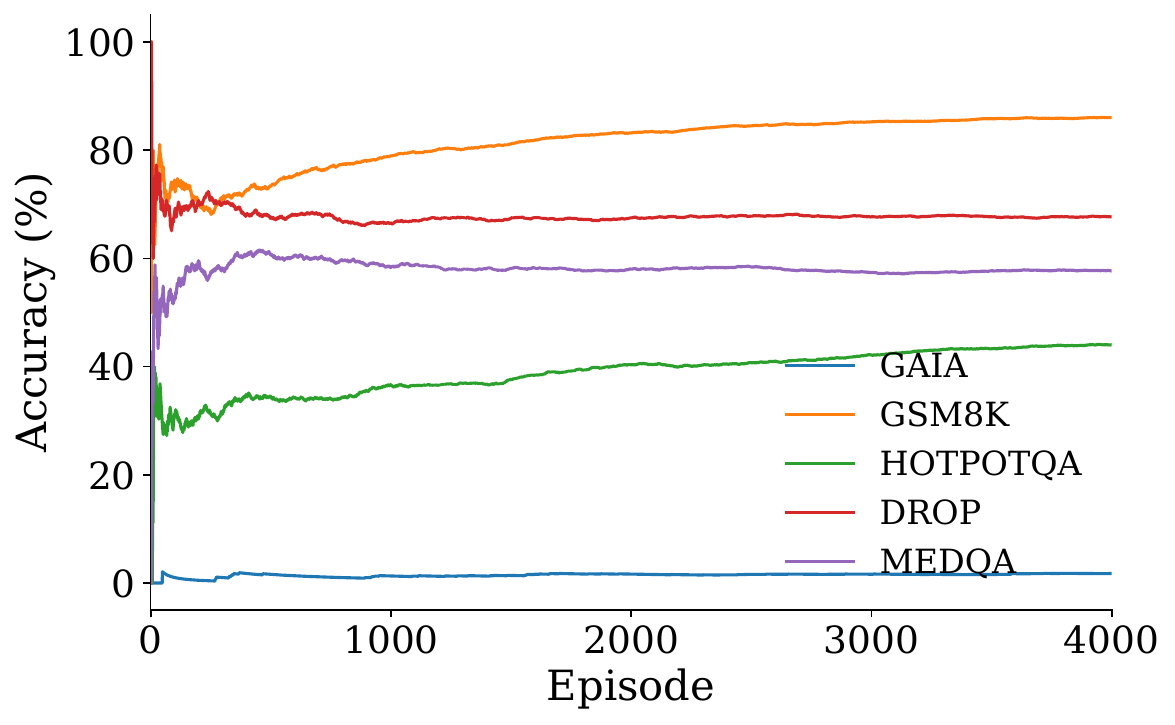}
    \caption{\textbf{Training dynamics of ARC across datasets.} Left: cumulative reward over episodes, showing steady improvement as the policy discovers higher-value configurations on GSM8K, DROP, MedQA, HotpotQA, and GAIA. Middle: rolling mean $\pm$ standard deviation of per-episode reward, indicating reduced variance and stabilization over time. Right: running validation accuracy, demonstrating that reward gains translate into improved task performance with dataset-specific convergence levels.}
    \label{fig:training_dynamics}
\end{figure}
\begin{wrapfigure}[21]{r}{0.5\linewidth}
    \centering
    \vspace{-1.0em}
    \includegraphics[width=\linewidth]{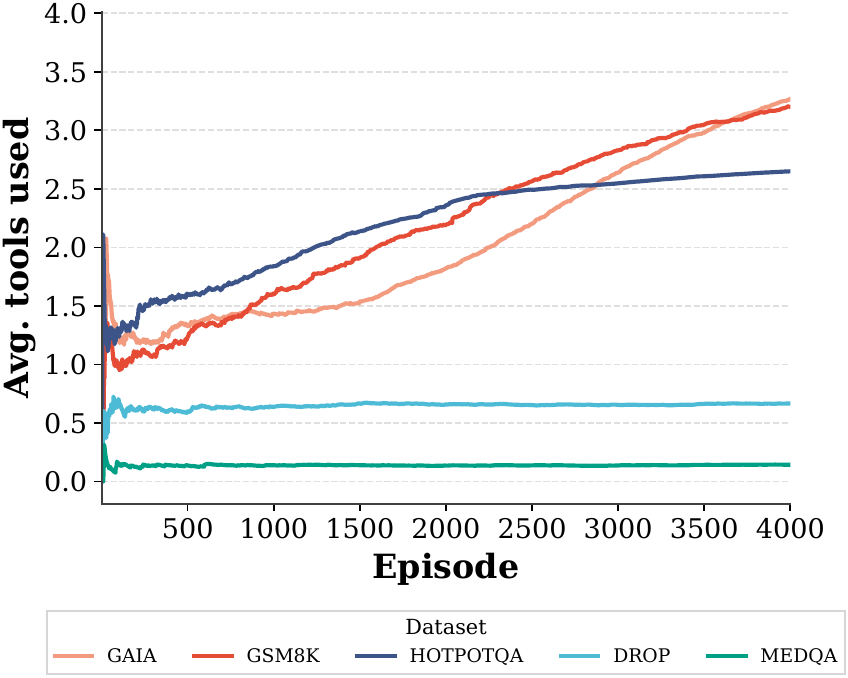}
    \caption{\textbf{Tool usage during training.} Running average number of tools used per episode for each dataset. ARC quickly learns sparse tool usage and gradually adjusts invocation patterns, with different steady-state levels reflecting task-specific reliance on tools.}
    \label{fig:tool-usage}
\end{wrapfigure}
Figure~\ref{fig:training_dynamics} summarizes training behavior across datasets. Cumulative and smoothed rewards increase steadily, while running accuracy converges to stable plateaus, indicating that the learned structure and prompt policies make consistent progress rather than overfitting to early episodes.

Figure~\ref{fig:workflow_evolution} illustrates how the structure policy’s workflow distribution evolves over the course of training. Early on, the agent explores a broad mix of patterns, but as learning proceeds the distribution sharpens and different datasets either converge to distinct dominant workflows or maintain a small mixture of high‑value patterns, reflecting task-dependent preferences (e.g., more verification-heavy patterns on GSM8K versus coordination-heavy patterns on GAIA). This behavior confirms that ARC is not merely memorizing a single “best” architecture, but actively specializing (and, when needed, mixing) structural choices to match the demands of each domain.

\begin{figure}[t]
    \centering
    \includegraphics[width=1\linewidth]{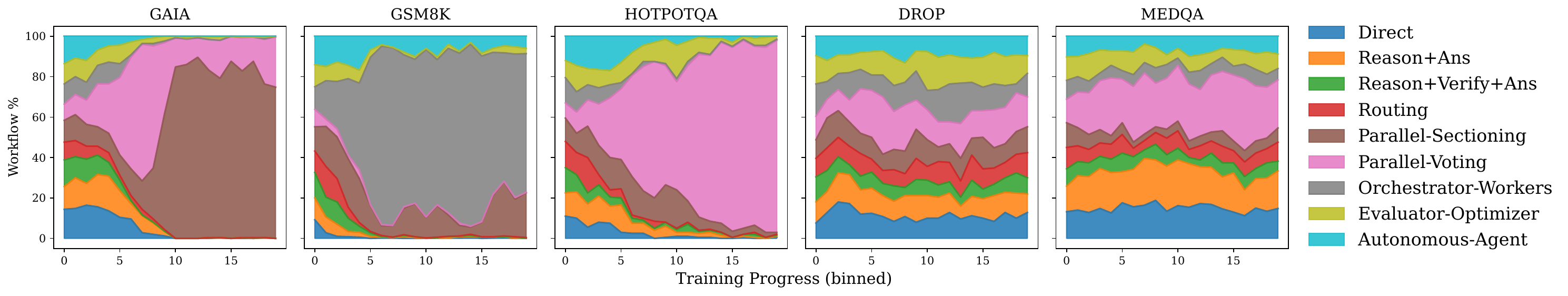}
    \caption{\textbf{Evolution of workflow selection during training.} Stacked area plots show, for GSM8K, HotpotQA, and GAIA, the fraction of episodes assigned to each workflow as training progresses. The structure policy quickly prunes suboptimal patterns and concentrates mass on a small set of task-appropriate workflows (e.g., Evaluator–Optimizer on GSM8K, Orchestrator–Workers on HotpotQA).}
    \label{fig:workflow_evolution}
\end{figure}

Figure~\ref{fig:tool-usage} tracks the running average number of tools invoked per episode during training. Across datasets, the policy initially explores a wider range of tool configurations and then converges to stable, task-specific usage levels, indicating that ARC learns when tools are actually helpful rather than indiscriminately calling them. Notably, tool-centric benchmarks such as HotpotQA and GAIA converge to higher usage than primarily textual reasoning tasks like GSM8K, suggesting that the learned structure policy adapts its reliance on tools to the demands of each domain.

\section{Additional Analysis}
\label{app:add_analysis}

\begin{figure}[h]
    \centering
    \includegraphics[width=1\linewidth]{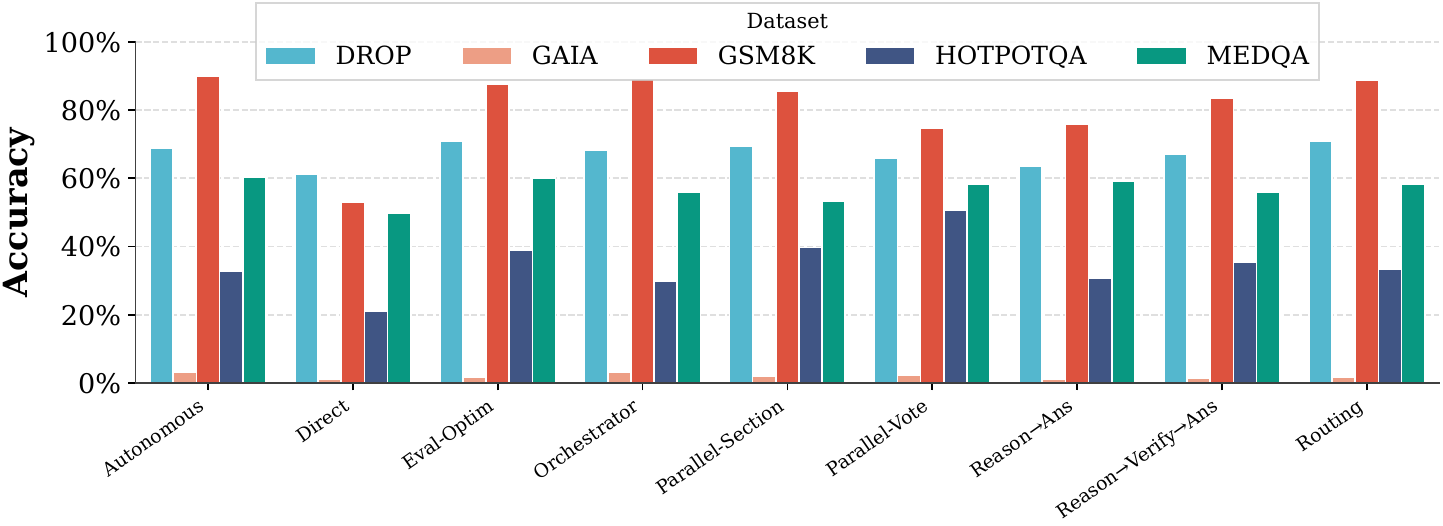}
    \vspace{-1em}
    \caption{\textbf{Accuracy by workflow and dataset.} Each bar shows the average accuracy of a fixed workflow on a given benchmark. Performance varies substantially across workflows and tasks no single workflow is uniformly optimal—highlighting the importance of learning query-adaptive configurations rather than relying on a fixed architecture.}
    \label{fig:workflow_acc}
\end{figure}

\begin{figure}[h]
    \centering
    \includegraphics[width=1\linewidth]{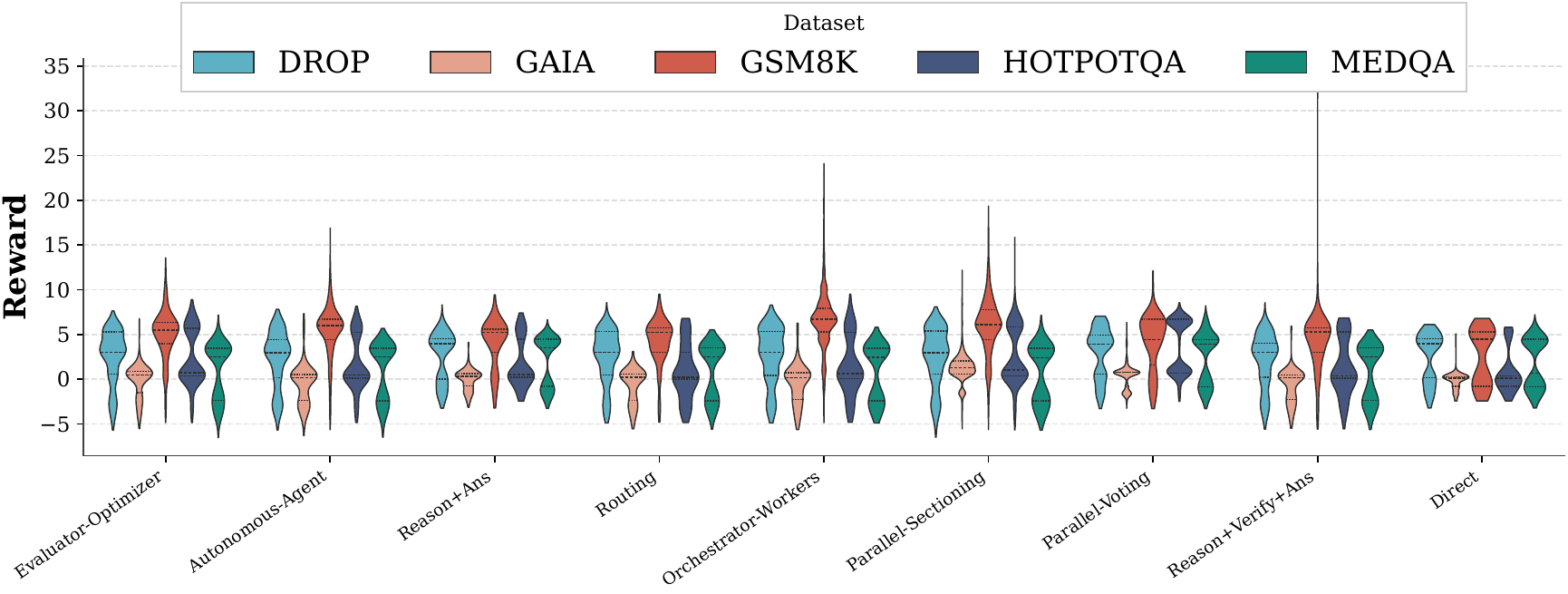}
    \caption{\textbf{Reward distribution by workflow.} Violin plots show the distribution of per-episode rewards for each workflow across datasets. Higher-performing workflows exhibit both higher central reward and tighter spread, illustrating that certain structural patterns not only achieve better returns but also yield more stable behavior during training.}
    \label{fig:reward_violin}
\end{figure}

Figure~\ref{fig:workflow_acc} breaks down accuracy by workflow and dataset, revealing how different structural choices contribute to performance. Within each benchmark, accuracy varies substantially across workflows, with multi-step patterns (e.g., Reason→Ans or Reason→Verify→Ans) typically outperforming simple Direct execution, and more complex coordination patterns (such as routing or parallel voting) being beneficial only on some tasks. No single workflow dominates across all datasets, underscoring that the optimal configuration is highly task-dependent and motivating the need for a learned policy that can adaptively select workflows rather than relying on a fixed template.

Figure~\ref{fig:reward_violin} visualizes the reward distribution for each workflow across datasets. We observe that workflows favored by the learned policy (e.g., Reason→Ans and Reason→Verify→Ans) concentrate mass at higher rewards with lower variance, whereas rarely selected patterns exhibit broader, lower-reward distributions, indicating that the policy systematically avoids structurally inefficient configurations.

\section{Alternative Training Objectives}
\label{app:training_alternatives}

We compared PPO against two alternative RL algorithms:

\textbf{GRPO (Group Relative Policy Optimization):} A variant of PPO that uses group-based advantage estimation to reduce variance. On GSM8K, GRPO achieved $81.2\%$ accuracy after 2,000 episodes, compared to PPO's $85.7\%$. GRPO struggled with the sparse binary reward signal from task correctness, as group normalization dampened learning signals.

We used the following hyperparameters for GRPO:
\begin{itemize}[noitemsep]
    \item Learning rate: $3 \times 10^{-4}$
    \item Batch size: 64 episodes
    \item Clip epsilon: $\epsilon = 0.2$
    \item Discount factor: $\gamma = 0.99$
    \item Entropy coefficient: $0.05$
    \item KL coefficient: $0.0$ (no KL regularization)
    \item Update epochs per batch: 4
    \item Max gradient norm: $0.5$
    \item Advantage computation: $A_i = \frac{R_i - \bar{R}}{\sigma_R + 10^{-8}}$ (group-relative)
\end{itemize}

\textbf{DPO (Direct Preference Optimization):} A preference-based method requiring pairwise configuration comparisons. We sampled pairs of configurations and labeled preferences based on reward differences. DPO achieved $79.8\%$ accuracy but required $3\times$ more environment interactions to collect pairwise data. Additionally, the preference labeling was noisy when configurations had similar rewards, leading to unstable training.

DPO hyperparameters:
\begin{itemize}[noitemsep]
    \item Learning rate (structure policy): $1 \times 10^{-4}$
    \item Learning rate (prompt policy): $1 \times 10^{-5}$
    \item Batch size: 16
    \item DPO temperature: $\beta = 0.05$
    \item Entropy coefficient: $0.05$
    \item Training epochs: 3
    \item Max gradient norm: $0.5$
    \item Preference pair thresholds: correct episodes with reward $\geq 4.0$, incorrect episodes with reward $\leq 2.0$
\end{itemize}

\section{Error Categorization Methodology}
\label{app:error_methodology}

We automatically classify errors into four categories based on heuristic analysis of the episode data. Each incorrect prediction is analyzed as follows:

\subsection{Policy Configuration Errors}

We detect policy misconfiguration by checking if the selected workflow, tools, and token budget are appropriate for the query:

\begin{itemize}[noitemsep]
    \item \textbf{Workflow mismatch}: Simple workflows (Direct) assigned to multi-step problems (detected via question length $>$100 words or presence of multiple sub-questions).
    \item \textbf{Tool mismatch}: Calculator missing when query contains arithmetic keywords (``calculate'', ``sum'', ``total''); web search missing when query asks about real-world facts.
    \item \textbf{Budget under-allocation}: Low token budget ($<$256) assigned to complex queries (word count $>$50 or contains ``explain'', ``step-by-step'').
\end{itemize}

\subsection{Reasoning Errors}

We detect reasoning failures by comparing the prediction structure against ground truth:

\begin{itemize}[noitemsep]
    \item \textbf{Wrong operation}: Prediction uses addition when ground truth uses subtraction (detected via operation keywords: ``add''/``plus'' vs ``subtract''/``minus''/``remaining'').
    \item \textbf{Missing steps}: Ground truth contains $\geq$3 reasoning steps but prediction is a single textbf.
    \item \textbf{Comprehension failure}: Query contains critical constraints (``remaining'', ``left'', ``after'') that are absent from prediction.
\end{itemize}

\subsection{Knowledge Gap Errors}

We detect knowledge/retrieval failures for tool-use tasks:

\begin{itemize}[noitemsep]
    \item \textbf{Retrieval failure}: Prediction contains phrases like ``could not find'', ``no information available'', ``cannot determine'' while using retrieval tools (web search, code execution).
    \item \textbf{Factual error}: Uses retrieval tools but provides a confident answer that differs from ground truth without explicit ``not found'' indicators (hallucination).
\end{itemize}

\subsection{Execution Errors}

We detect execution failures where the approach is correct but output is wrong:

\begin{itemize}[noitemsep]
    \item \textbf{Arithmetic error}: Ground truth contains intermediate calculations (e.g., $\langle\langle$5*3=15$\rangle\rangle$) and prediction produces a different numeric answer.
    \item \textbf{Answer extraction error}: Correct answer appears in prediction text but a different value is extracted as the final answer.
\end{itemize}

\textbf{Priority.} Errors are assigned to the first matching category in order: policy configuration $\rightarrow$ answer extraction $\rightarrow$ reasoning $\rightarrow$ arithmetic $\rightarrow$ knowledge gap $\rightarrow$ unclassified. This ensures policy failures are surfaced first, as they are most actionable for our framework.

\subsection{Failure Case Examples}
\label{app:failures}

We provide concrete examples illustrating each error category:

\textbf{Policy Configuration Error:}
\begin{itemize}[noitemsep]
    \item Query: ``Find $x$ such that $\log_2(x) + \log_2(x-7) = 3$''
    \item Selected: Direct workflow, Low budget
    \item Issue: Multi-step algebra requires Reason+Verify workflow with high budget.
\end{itemize}

\textbf{Reasoning Error:}
\begin{itemize}[noitemsep]
    \item Query: ``John has 5 apples. He gives 2 to Mary. How many does John have left?''
    \item Prediction: ``John has 5 + 2 = 7 apples.''
    \item Issue: Used addition instead of subtraction despite ``gives'' and ``left'' keywords.
\end{itemize}

\textbf{Knowledge Gap Error:}
\begin{itemize}[noitemsep]
    \item Query: ``Who directed the 2014 film Big Stone Gap?''
    \item Prediction: ``Based on the search results, I cannot find information about the director.''
    \item Ground truth: ``Adriana Trigiani''
    \item Issue: Retrieval failed to find available information.
\end{itemize}

\textbf{Execution Error:}
\begin{itemize}[noitemsep]
    \item Query: ``What is $\frac{1}{2} \times 5 \times 12$?''
    \item Prediction: ``$\frac{1}{2} \times 5 \times 12 = 25$''
    \item Ground truth: 30
    \item Issue: Correct formula, wrong arithmetic.
\end{itemize}